\documentclass[journal]{IEEEtai}

\usepackage[colorlinks,urlcolor=blue,linkcolor=blue,citecolor=blue]{hyperref}

\usepackage{color,array}

\usepackage{cite}
\usepackage{amsmath,amssymb,amsfonts}
\usepackage{amsthm}
\newtheorem{assumption}{Assumption}
\newtheorem{theorem}{Theorem}
\newtheorem{lemma}{Lemma}
\newtheorem{remark}{Remark}
\usepackage{algorithm, algorithmic}
\usepackage{graphicx}
\usepackage{float}
\usepackage{subcaption}
\usepackage{textcomp}
\usepackage{xcolor}
\usepackage{booktabs}
\usepackage{multirow}
\usepackage[normalem]{ulem}
\usepackage[T1]{fontenc}
\useunder{\uline}{\ul}{}

\begin{document}

\title{FedGMI: Generative Model-Driven Federated Learning for Probabilistic Mixture Inference} 

\author{\IEEEauthorblockN{
            Qijun Hou\IEEEauthorrefmark{2}, Yuchen Shi\IEEEauthorrefmark{2}, Pingyi Fan\IEEEauthorrefmark{1}\IEEEauthorrefmark{2},~\IEEEmembership{Senior Member,~IEEE,} and Khaled B. Letaief\IEEEauthorrefmark{4},~\IEEEmembership{Fellow,~IEEE} }
        
	\IEEEauthorblockA{
            \IEEEauthorrefmark{2}Dept. of Electronic Engineering, BNRist, Tsinghua University.\\
            \IEEEauthorrefmark{4}Dept. of Electronic and Computer Engineering, HKUST.}
    
        \thanks{\IEEEauthorrefmark{1}Pingyi Fan is the corresponding author, e-mail: fpy@tsinghua.edu.cn. This work was supported by the National Key Research and Development Program of China (Grant NO.2021YFA1000500(4)).}
        \thanks{\IEEEauthorrefmark{4}Khaled B. Letaief was supported in part by Hong Kong Research Grants Council under the Areas of Excellence scheme grant AoE/E-601/22-R.}
        }

\maketitle

\begin{abstract} 
Federated Learning (FL) facilitates collaborative model training across decentralized clients while preserving data privacy by avoiding raw data exchange. Despite its potential, FL performance is often compromised by data heterogeneity across clients.
To address this, Clustered Federated Learning (CFL) groups clients with similar data distributions to improve model performance, but constrained by intra-cluster heterogeneity. Conversely, Personalized Federated Learning (PFL) tailors models to individual clients, but usually neglects the underlying structural similarities among clients. 
In this work, we investigate a probabilistic mixture (PM) scenario, where each client's local data distribution is modeled as a convex combination of several shared inherent distributions. 
To effectively model this structure, we propose FedGMI, a framework that utilizes Variational Autoencoders (VAEs) as generative density estimators to represent these inherent distributions and infer the mixture components of clients' local data distributions.
This approach enables structured personalization without sacrificing the benefits of collaborative learning.
Extensive experiments demonstrate that FedGMI effectively characterizes and discriminate the inherent distributions, as well as accurately estimates mixture proportions. Furthermore, FedGMI maintains robust performance even under communication cost constraints.
\end{abstract}

\begin{IEEEImpStatement}
Federated Learning (FL) enables collaborative training across clients, preserving data privacy. However, heterogeneous data distributions influence the model performance significantly. Existing methods, such as Clustered FL and Personalized FL, struggle with intra-cluster variability and fail to capture client similarities. FedGMI introduces a generative model-driven approach using Variational Autoencoders (VAEs) to represent inherent data distributions and estimate client-specific mixtures. This enables structured personalization while leveraging collaborative learning. FedGMI's effectiveness under low communication costs enhances scalability and deployability on resource-constrained edge devices. Applications span healthcare (diverse patient data across hospitals), smart devices (personalized user patterns), and finance (regional fraud detection and prevention).
\end{IEEEImpStatement}


\section{Introduction}

\IEEEPARstart{F}{ederated Learning} (FL) is an emerging machine learning paradigm that facilitates collaborative intelligence across decentralized data sources while preserving data privacy and security. Unlike traditional machine learning approaches, which aggregate data in a centralized server, FL enables clients (mobile devices, edge servers, etc.) to collaboratively optimize a shared model without exposing their local data. 
Since FedAvg\cite{ref:fedavg} provided a functional framework of model aggregation, FL has garnered substantial attention. The rise of edge computing and the growing concerns about data privacy have further accelerated the adoption of FL in real-world scenarios.

Despite its potential, the practical use of FL is often limited by data heterogeneity, as the data distributions across clients are typically non-independent and identically distributed (non-IID). Non-IID data can lead to issues such as biased model updates, diminished generalization, and slower convergence. 
A single global model not only suffers from performance degradation in the presence of non-IID data, but also fails to capture the underlying structural relationships among clients. 
To address this, Clustered-FL (CFL) \cite{ref:cfl}\cite{ref:ifca} clusters homogeneous clients based on their data distributions and trains models for each cluster to make training more stable. Personalized-FL (PFL) \cite{ref:pfl-dkd} focuses on fine-tuning models for each individual client. 
However, there is a clear trade-off: While CFL focuses on common properties across clients within the same cluster, it may overlook the unique characteristics and preferences of individual clients within a cluster. Conversely, PFL focuses on single clients but often neglects the structural similarities or homogeneity across clients.

In essence, \textit{Clustering} and \textit{Personalization} represent two complementary perspectives on modeling the relationships among client distributions: the former exploits similarities, while the latter preserves individual differences. PFL can build on clustering to tailor models for client-specific model adaptation. A growing body of work further blurs the distinction between these two paradigms, leveraging probabilistic models to more precisely measure the relationships among client distributions.

Studies suggest that real-world data is likely to follow a mixture paradigm, where each client's local distribution is a combination of several shared inherent distributions \cite{ref:fedem}\cite{ref:fedce}\cite{ref:fedsoft}\cite{ref:fedgmm}. We refer to this data composition scenario as \textbf{Probabilistic Mixture} (\textbf{PM}), formally defined as Assumption.\ref{asp:pm}. A clear example can be found in federated image classification on handwritten digits (e.g., MNIST). The underlying inherent distributions could be distinct fundamental writing styles, such as 'slanted', 'neat', or 'broad'. The local data of a single client could be a mixture of these styles. Our work aims to leverage this mixture property to integrate the strengths of CFL and PFL. Through learning these shared inherent distributions and estimating each client's specific mixture proportions, we provides insights into both the structural similarities and the unique characteristics of each client. 

Conventional approaches for measuring inter-client similarity rely on fit the distributions of their data (original data or features) to parametric models (e.g., \textit{Gaussian}, \textit{Bernoulli}). However, explicitly modeling complex, high-dimensional modalities, such as images, remains intractable and limits the applicability of these approaches. To address this, generative models, such as Variational Autoencoders (VAEs) and Generative Adversarial Networks (GANs), have gained considerable attention in FL due to their capacity to represent complex data distributions. 
By leveraging their ability to estimate probability density and distribution divergence (e.g., \textit{KL-divergence}), generative models provide a more flexible way to analyze client data relationships. 

In this paper, we propose FedGMI, a mixture-aware Federated Learning (FL) framework designed to enhance model performance by identifying and adapting to the underlying data structures in the \textbf{PM} situation. FedGMI utilizes VAEs to represent heterogeneous data distributions and employs a modified Maximum A Posteriori (MAP) criterion to partition local data samples. 
Unlike standard FL, FedGMI jointly optimizes: (\textit{i}) VAE-based generative models representing inherent distributions, (\textit{ii}) Distribution-Expert Classifiers (e.g., CNNs) for each inherent distribution, and (\textit{iii}) estimation of each client's mixture proportions.
FedGMI is tested on multiple common datasets. The results validate that FedGMI effectively learns distinctive inherent distributions and estimates the mixture structure of client data distributions.

The remainder of this paper is organized as follows: Section \ref{sec:RW} provides a review of related works on generative models in Federated Learning, with emphasis on probabilistic mixture approaches. In Section \ref{sec:MT}, we present our proposed method, FedGMI, along with its corresponding algorithm, which leverages Variational Autoencoders (VAEs) to partition local data. Section \ref{sec:CA} provides the theoretical foundation by establishing the convergence properties of our method. In Section \ref{sec:EX}, we conduct experiments on the benchmark datasets to assess the effectiveness and performance of our approach. The paper concludes with a summary of findings and potential directions for future research.

\section{Related Work}
\label{sec:RW}
Since the concept of FedAvg was proposed in \cite{ref:fedavg}, FL has been studied as a framework of collaborative learning with privacy protection.
Extensive research has focused on the performance degradation of FL caused by data heterogeneity \cite{ref:avgconv}\cite{ref:flrl}.
While most FL frameworks are designed to be task-agnostic, their efficacy is typically demonstrated on supervised learning tasks, most notably image classification. 
To this end, various strategies have been proposed: VHFL\cite{ref:vhfl} proposed a framework to improve the performance by combining global observation with distributed data. ISFL\cite{ref:isfl} and SAM\cite{ref:sam} leverage sampling mechanisms to mitigate the impact of distribution shift. Architectural adjustments such as FedLP\cite{ref:fedlp} and FedNC\cite{ref:fednc} utilize pruning and network coding to enhance the communication efficiency in non-IID data settings. 
Beyond optimization, probabilistic perspectives have also gained traction; for instance, Louizos et al. \cite{ref:empers} explored the influence of parameter priors using a latent variable approach.

The quantified measurement of inter-client heterogeneity and similarity is pivotal to CFL\&PFL. Some researchers utilized geometric metrics in the model parameter space, such as the cosine similarity \cite{ref:cfl}. Building on this, Briggs et al. (2020) performed hierarchical clustering on clients based on the distance between their parameter gradients \cite{ref:hiercls}. Beyond parameter-based metrics, IFCA clusters clients using the loss function as a measure of heterogeneity\cite{ref:ifca}. Recently, probabilistic perspectives have gained prominence. Zhang et al. (2023) and Wu et al. (2024) utilized Bayesian methodology, which combines the features of CFL and PFL, to improve clustering in the FL framework \cite{ref:fedpsc}\cite{ref:fedbay}.

Previous research has studied the \textbf{PM} situation from several perspectives. Marfoq et al. (2022) presented FedEM, which treats each local data distributions as a mixture of latent inherent distributions, and theoretically validated the assumption\cite{ref:fedem}. While FedEM employs a regularized optimization objective to capture inter-client relationships, it has been noted for instability in proportion estimation.
FedSoft \cite{ref:fedsoft} and FedCE \cite{ref:fedce} enable soft clustering by incorporating mixture coefficients directly into the local training objective. FedCE also incorporated historical associations into the global aggregation. These frameworks allow clients to associate with multiple distributions, effectively balancing personalization with clustering. These works couple the estimation of proportions with the training classification tasks, and acquire classifiers for each inherent distributions respectively.
FedGMM models the local data distribution of each client with a Gaussian Mixture Model (GMM), leveraging EM algorithm to jointly optimize the parameters of the Gaussian components and their mixture weights\cite{ref:fedgmm}. However, the modeling capacity of GMM for high-dimensional data is limited. Although existing methods above succeeded in capturing the mixture structures, more explicit characterization of the inherent distributions, e.g., probability density estimation, still requires further research.

Generative models are widely used to enhance performance and tackle heterogeneity in FL. FedGAN extends the decentralized model aggregation method to generative models and proves its convergence and effectiveness\cite{ref:fedgan}. FedCG trains Conditional-GAN(C-GAN) for each client to learn the semantics of local data distributions, and the server aggregates the parameters of generators instead of CNN classifiers. The result shows the capability of generative models functioning as a semantic base and representing probabilistic distributions of data sources\cite{ref:fedcg}. FedVGL has clients who possess a generative model and classifier and use VAEs to improve the classification\cite{ref:fedvgl}. The superior performance of these methods --- especially under severe non-IID conditions --- highlights the potential of generative modeling within the FL context, providing a strong motivation for our VAE-driven approach.

\section{Methodology}
\label{sec:MT}
\begin{figure*}[ht]
    \centering
    \includegraphics[width=0.85\textwidth]{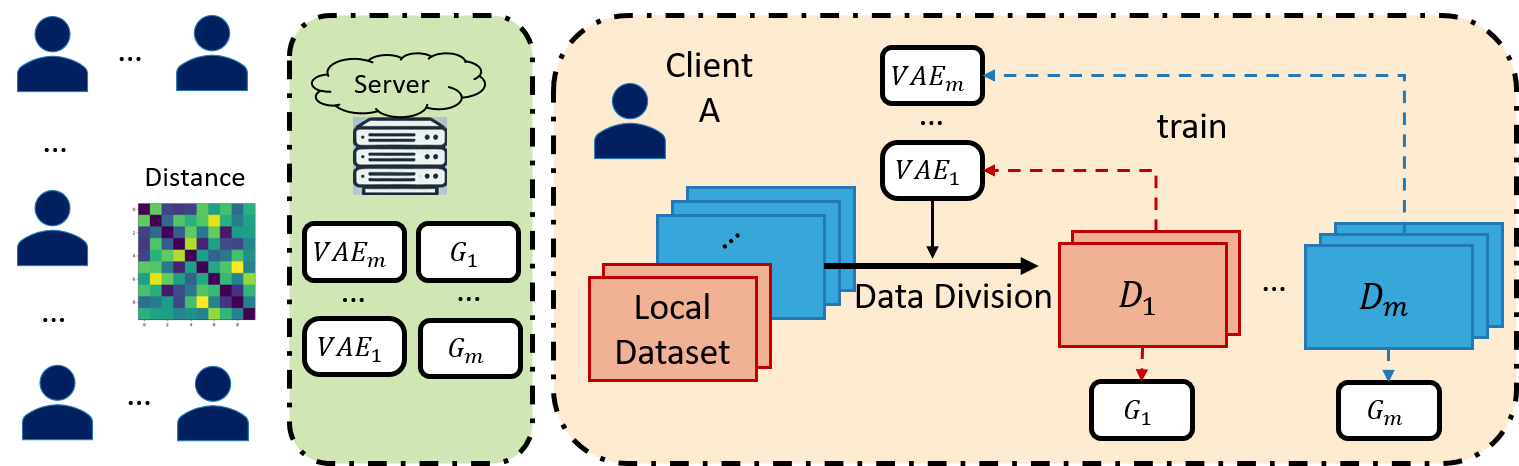}
    \caption{\textbf{Architectural overview of the FedGMI framework.} Local datasets are modeled as mixtures of inherent distributions (represented by distinct colors). Data division is supposed to partition the samples by their corresponding distribution. The subsets iteratively train the VAEs and the classifiers with their original loss functions. Server collects and aggregate the weights of both the VAEs and the classifiers. A distance metric is used to initialize.}
    \label{fig:system}
\end{figure*}
\subsection{Problem Definition}

Consider a federated setting with $N$ clients. Each client $C_i$ possesses a local dataset $D_i$, sampled from a local distribution $p_i$, $i = 0, 1, \dots, N-1$. We assume that each $D_i$ is a mixture of samples from $M$ shared inherent distributions, denoted as $\{q_j\}_{j=0}^{M-1}$. The \textbf{PM} assumption can be formally defined as:
\begin{assumption}
    \label{asp:pm}
    (Probabilistic Mixture) For $i = 0, 1, \dots, N-1$ and \(j = 0,1,\dots, M-1\), \(\exists \{\alpha_i^j\}\) that satisfy:
    \begin{equation}
        \begin{aligned}
            p_i &= \sum_{j=0}^{M-1} \alpha_i^j q_j \\
            \text{where}\quad \alpha_i^j &\geq 0 , \quad\sum_{j=0}^{M-1} \alpha_i^j = 1
        \end{aligned}
    \end{equation}
\end{assumption}

While the Federated Learning (FL) framework is task-agnostic, this paper considers the classification task as a representative example. We denote the loss function for a classification model $G$ on a data sample $(\mathbf{x}, y)$ as $L(\mathbf{x}, y; G)$.  A Variational Autoencoder (VAE), denoted as $\phi_{p}=\mathcal{D}_{p}\circ \mathcal{E}_{p}$, is used to characterize the distribution $p$, where $\mathcal{E}, \mathcal{D}$ represent the encoder and decoder.

To measure the extent to which a sample $(\mathbf{x}, y)$ correspond to a distribution $q$, we introduce an affinity score function $l(\mathbf{x}, y; q)$ satisfying:
(\textit{i}) \(\sum_{j=0}^{M-1}l(\mathbf{x}, y; q) = 1\);
(\textit{ii}) A higher value indicates a greater likelihood that $(\mathbf{x}, y)$ belongs to $q$. 
We will show in section \ref{sec:vae}, \ref{sec:div} that $l(\mathbf{x}, y; q)$ is analytically related to the VAE loss function $\mathcal{L}(\mathbf{x}, \phi_q)$. 

Thus, the purpose of FedGMI is to optimize the following objectives simultaneously over \(T\) communication rounds:

\paragraph{Distribution-Expert Classifiers for inherent distributions}: 
\begin{equation}
    \begin{aligned}
        \Bigl\{G_{q_j} = \operatorname*{arg\,min}_{G} \mathbb{E}_{(\mathbf{x}, y)\sim q_j}L(\mathbf{x}, y; G) \mid\\
        j = &0, \dots, M-1\Bigr\}
    \end{aligned}
\end{equation}
\paragraph{Probabilistic Mixture Inference for clients}:
\begin{equation}
    \begin{aligned}
        &\Bigl\{\alpha_i^j \mid \alpha_i^j \geq 0 , \sum_{j=0}^{M-1} \alpha_i^j = 1,\\
        &(\alpha_i^0, \dots, \alpha_i^{M-1}) = \operatorname*{arg\,min}_{\alpha^0, \dots, \alpha^{M-1}} \mathbb{E}_{(\mathbf{x}, y)\sim p_i} L(\mathbf{x}, y; G_{q_{\mathbf{x}, y}^*})\\
        &i = 0,\dots,N-1, j = 0, \dots, M-1 \Bigr\}
    \end{aligned}
\end{equation}

$\mathbb{E}$ denotes the expectation. \(q_{\mathbf{x}, y}^*\) denotes the inherent distribution with the highest affinity score for \((\mathbf{x}, y)\). The index $j \in \{0, \dots, M-1\}$ of \(q_{\mathbf{x}, y}^*\) is then denoted as $s^*_{\mathbf{x},y}$.

\subsection{VAEs as probability density estimators}
\label{sec:vae}
In this paper, we use VAEs to represent heterogeneous data distributions since the probability density can be approximately estimated by the VAE loss function. For a given set of input $\{\mathbf{x}\}$ sampled from distribution $p(\mathbf{x})$, the standard VAE optimization objective is formulated as:
\begin{equation}
    \min \mathbb{E}_{\mathbf{x} \sim p(\mathbf{x})}\left[-\log P_{\mathcal{D}}(\mathbf{x})\right]
\end{equation}

Optimizing this objective drives the decoder's generative distribution $P_{\mathcal{D}}(\mathbf{x})$ to approximate the data distribution $p(\mathbf{x})$, since the expected negative log-likelihood in the objective equals the KL-divergence between $P_{\mathcal{D}}(\mathbf{x})$ and $p(\mathbf{x})$.  
Following the derivation of the Evidence Lower Bound (ELBO)\cite{ref:elbo}, as shown in eq.\eqref{elbo:1}\eqref{elbo:2}, the negative log probability can be decomposed as the VAE loss function $\mathcal{L}(\mathbf{x};\phi_{p})$ minus a KL-divergence term which converges to a constant value as the training progresses. 
The VAE loss function $\mathcal{L}(\mathbf{x};\phi_{p})$ consists of the reconstruction error \(\mathcal{L}_{rec}\) and the KL-divergence error \(\mathcal{L}_{\text{KL}}\), as shown in eq.\eqref{elbo:2}. 
\begin{figure}[htbp]
    \centering
    \begin{align}
        -\log P_{\mathcal{D}}(x)
        &= \mathbb{E}\left[-\log \frac{P_{\mathcal{D}}(\mathbf{x}|z)P(z)}{P_{\mathcal{E}}(z|\mathbf{x})}\right] \nonumber\\[6pt]
        &\quad - D_{\text{KL}}(P_{\mathcal{E}}(z|\mathbf{x}) \| P_{\mathcal{D}}(z|\mathbf{x}))  \label{elbo:1}\\[6pt]
        &= \underbrace{\mathbb{E}\left[-\log P_{\mathcal{D}}(\mathbf{x}|z)\right]}_{\mathcal{L}_{rec}} + \underbrace{D_{\text{KL}}(P_{\mathcal{E}}(z|\mathbf{x}) \| P(z))}_{\mathcal{L}_{\text{KL}}} \nonumber \\[6pt]
        &\quad - \underbrace{D_{\text{KL}}(P_{\mathcal{E}}(z|\mathbf{x}) \| P_{\mathcal{D}}(z|\mathbf{x}))}_{\approx c}  \label{elbo:2}\\[3pt]
        &\approx \mathcal{L}_{rec} + \mathcal{L}_{\text{KL}} - c = \mathcal{L}(\mathbf{x};\phi_{p}) - c  
    \end{align}
\end{figure}

Motivated by the dual capabilities of VAEs, which act as both generative models that enable efficient sampling and as density estimators that can approximate log probabilities, we employ multiple VAEs to represent the local client distributions $\{p_i\}_{i=0}^{N-1}$ and the inherent distributions $\{q_j\}_{j=0}^{M-1}$ respectively.

\subsection{Local Data Division}
\label{sec:div}
In this paper, we determine which inherent distribution a given sample $x$ belongs to by matching the one with the highest affinity score. To construct a affinity score $l(\mathbf{x}; q)$ that is monotonically increasing with the fit between $\mathbf{x}$ and a distribution $q$, we adopt the Maximum A Posteriori (MAP) principle and define $l(\cdot)$ as the product of the likelihood and the prior, which is proportional to the posterior probability: 
\begin{equation}
    \label{def:l}
    \begin{aligned}
        l(\mathbf{x}, y; q_j) &= P(\mathbf{x}, y\sim q_j | \mathbf{x}, y) \\
        &\propto P(\mathbf{x}, y|\mathbf{x}, y\sim q_j)\cdot P(\mathbf{x}, y\sim q_j)
    \end{aligned}
\end{equation}

The prior probability $P(\mathbf{x}, y\sim q_j)$ for client $C_i$ is estimated as the proportion of the local data samples currently assigned to the inherent distribution $q_j$, noted as $\mathcal{S}(i; j) = \{\mathbf{x}, y \in D_i|s_{\mathbf{x}, y}^{*} = j\}$, relative to the total local dataset size $|D_i|$. 
These priors are initialized to a uniform distribution $\mathcal{U}$ and will be updated iteratively during training. The log likelihood probability $P(\mathbf{x}, y|\mathbf{x}, y \sim q_{j})$ is estimated by $-\mathcal{L}(\mathbf{x}, y;\phi_{q_j})$. 
Since what we are concerned with is the ordering (i.e., the ranking) rather than the absolute magnitudes, the softmax function can effectively transform these log values to a normalized probability vector while preserving the ordinal information. The empirical affinity score $l(\mathbf{x}, y; q_{j})$ can be calculated as the $j$\textsuperscript{th} entry of a vector with length $M$:
\begin{equation}
    \label{def:l2}
    \begin{aligned}
        &\mathbf{l}^{(t+1)}(\mathbf{x}, y)  \\ &= \text{softmax}\left(\left[-\mathcal{L}(\mathbf{x}, y;\phi^{(t)}_{q_j})\right]_{j=1}^{M}\right)\odot \left[\frac{|\mathcal{S}^{(t)}(i; j)|}{|D_i|}\right]_{j=1}^{M}
    \end{aligned}
\end{equation}
where $t$ denotes the communication round and $\odot$ represents the element-wise product. 

Notably, clients maintain this partitioning locally. This design eliminates the overhead of downloading model parameters for all inherent distributions in every round, except when the affinity scores require update. Such a decoupled architecture allows the communication cost to be flexibly adjusted by tuning the update frequency, making FedGMI highly adaptable to various communication constraints.

\subsection{FedGMI: Algorithm}
In this section, The FedGMI procedure is formally outlined in Algorithm \ref{alg:FedGMI}. Before the communication rounds start, each client prepares a locally trained VAE  $\phi_{p_i}$ and uploads it to the server. 
To ensure stable convergence, VAEs for inherent distributions $q_j$ are initialized in the server by a greedy algorithm (Algorithm \ref{alg:init}) rather than random initialization. This is critical because the randomly initialized VAEs usually give unstable data division. To intuitively explain this phenomenon, data from different clients have different distributions but also share common properties. VAEs that capture both the common properties and the different distributions give the ideal division in our considered situation. Randomly initialized VAEs may differ in their ability to extract common properties, which causes most samples to be assigned to the single subset represented by the VAE with the best ability to extract common properties, resulting in a trivial partitioning. This is also discussed in \cite{ref:fedsoft}. 
Mathematically, random initialization is unlikely to satisfy Assumption \ref{ast:delta}, which is essential for the convergence of the data division process as analyzed in Section \ref{sec:CA}.
To address this, Algorithm \ref{alg:init} greedily selects $M$ VAEs that maximize the pairwise distance. The distance between two VAEs $\phi_{q_i}$ and $\phi_{q_j}$ is measured by the empirical KL-divergence:
\begin{equation}
\label{eq:KL}
    \begin{aligned}
        D_{\text{KL}}(q_i\|q_j) &= \int_{\mathcal{X}} q_{i}(x)\log \frac{q_{i}(x)}{q_{j}(x)}dx \\
        &= -\int_{\mathcal{X}} q_{i}(x)\log q_{j}(x)dx + \int_{\mathcal{X}} q_{i}(x)\log q_{i}(x)dx\\
        &= \mathbb{E}_{x\sim q_i}\left[-\log q_{j}(x)\right] - \mathbb{E}_{x\sim q_i}\left[-\log q_{i}(x)\right]\\
        &\approx  \frac{1}{B}\sum_{b=1}^{B}\mathcal{L}(\hat{x}_b;\phi_{q_j})-\mathcal{L}(\hat{x}_b;\phi_{q_i})\\
        \overset{z_b\sim \mathcal{N}(0;I)}{=}& \frac{1}{B}\sum_{b=1}^{B}\mathcal{L}\left(\mathcal{D}_{q_i}(z_b);\phi_{q_j}\right)-\mathcal{L}\left(\mathcal{D}_{q_i}(z_b);\phi_{q_i}\right)
    \end{aligned}
\end{equation}
The local data division is initialized at the beginning and updates every $\tau$ rounds. During intermediate rounds, each user can adaptively participate in updating the VAEs, the distribution-expert classifiers, or both, depending on their local computational and communication resources. The server selects clients randomly and aggregate the updated model parameters based on weights $\mathbf{\beta}$ calculated by eq.\eqref{eq:beta}. 
\begin{equation}
\label{eq:beta}
    \beta_{i}^{j} = \frac{\alpha_{i}^{j}\cdot |D_i|}{\sum_{i}\alpha_{t}^{j}\cdot |D_i|}
\end{equation}

\begin{algorithm}
\caption{FedGMI}
\label{alg:FedGMI}
\renewcommand{\algorithmicrequire}{\textbf{Input:}}
\renewcommand{\algorithmicensure}{\textbf{Output:}}
    \begin{algorithmic}
        \REQUIRE Global epoch $T$, division update interval $\tau$, number of clients $n$, client selection size $K$;
        \ENSURE Inherent VAEs $\phi_{q_1}, \dots \phi_{q_m}$, classifiers $G_{q_1}, \dots G_{q_m}$, ratio coefficients $\{\alpha\}_{N\times M}$;
        \vspace{0.2cm}
        
        \FORALL {t = 0, \dots, T-1}
            \IF{$t=0$} 
                \STATE All clients send local VAE $\{\phi_{p_i}\}_{i=1}^{n}$ to server
                \STATE $\phi_{q_1}, \dots \phi_{q_m} \leftarrow \text{\textit{Stable-Initialize}}\left(\{\phi_{p_i}\}_{i=1}^{n}\right)$
            \ENDIF
            \IF{$t\quad mod\quad  \tau = 0$} 
                \STATE Server send inherent VAEs $\phi_{q_1}, \dots \phi_{q_m}$ to all clients
                \STATE All clients update local division by \eqref{def:l}\eqref{def:l2}
            \ENDIF
            \STATE Server calculate $\beta_i^j$ based on \eqref{eq:beta}
            \STATE Server select a set of $K$ clients randomly
            \STATE Selected clients download $\{\phi_{q_j}\}$ and classifier $\{G_{q_j}\}$ and optimize the parameters by gradient descent method\\
            $\phi_{k}^{*} \leftarrow \operatorname*{arg\,min}_{\mathbf{x}\in \mathcal{S}(D_i; q_s)} \mathcal{L}(\mathbf{x};\phi)$\\
            $G_{k}^{*} \leftarrow \operatorname*{arg\,min}_{\mathbf{x}\in \mathcal{S}(D_i; q_s)} L(G(\mathbf{x}),y)$
            \STATE Client upload the updated parameters and Server aggregate the parameter
            \begin{align*}
                &\phi_{q_j}^{(t+1)}, G_{q_j}^{(t+1)} \leftarrow \\  &\sum_{k\in selected}\frac{\beta_k}{\sum_k\beta_k} \phi_{k}^{*}, \sum_{k\in selected}\frac{\beta_k}{\sum_k\beta_k} G_{k}^{*}
            \end{align*}
        \ENDFOR
    \end{algorithmic}
\end{algorithm}

\begin{algorithm}
\caption{Stable-Initialize}
\label{alg:init}
\renewcommand{\algorithmicrequire}{\textbf{Input:}}
\renewcommand{\algorithmicensure}{\textbf{Output:}}
    \begin{algorithmic}
        \REQUIRE $m$, VAEs $\phi_{1}, \dots \phi_{s}$ ($s\geq m \geq2$)
        \ENSURE initial VAE parameters $\phi_{(1)}, \dots \phi_{(m)}$
        \vspace{0.2cm}
        \STATE Initialize parameter set $\mathcal{I} = \{\}$\\
        Calculate the KL-divergence Matrix $\{D_{ij}\} = D_{\text{KL}}(q_i\|q_j)$
        \WHILE {$|\mathcal{I}| < m$}
            \STATE Add $p, q = \operatorname*{arg\,max} D_{ij}$ into $\mathcal{I}$ if $\mathcal{I}$ is empty
            \STATE Or add $p = \operatorname*{arg\,max}_{i \notin \mathcal{I}} \min_{j\in \mathcal{I}} D_{ij}$ if $\mathcal{I}$ is not empty
        \ENDWHILE
        \STATE Return parameters of $\phi_{p}$ for p in $\mathcal{I}$
    \end{algorithmic}
\end{algorithm}

\section{Convergence Analysis}
\label{sec:CA}
This section analyzes the convergence properties of the proposed FedGMI algorithm. The proof framework follows the approaches in \cite{ref:ifca} and \cite{ref:avgconv}. FedGMI consists of two fundamental operations: \textit{Data Division} and \textit{Federated Training}. 
The convergence of \textit{Data Division} is demonstrated in this section. 
\textit{Data Division} involves sample assignments based on eq.\eqref{def:l}\eqref{def:l2}, while \textit{Federated Training} involves training classification models on the partitioned data. 
Subsequently, in the \textit{Federated Training} stage, each classification model for an inherent distribution can be considered a \textit{near-IID} FedAvg process, as the data has been divided into identified distributions, which limits the extent of \textit{non-iid}. $\theta, \omega$ represents the parameter vector of VAEs and classification networks. 

To establish the convergence of \textit{Data Division}, we assume that our adopted loss functions satisfy Assumption \ref{ast:lip}, \ref{ast:convex}, which is common in FL analysis. Then we make Assumption \ref{ast:delta}, revealing the necessity of Stable-Initialization, similarly to \cite{ref:ifca}. Without this initialization constraint, it is hard to obtain a useful upper bound of the error rate of data division in Lemma \ref{lemma:Perr}. We also assume that the variance of the empirical loss function for any sample $x$ and its gradient is bounded by constants $\eta^2, \nu^2$. Here the extra Assumption \ref{ast:smooth} is used to bound the deviation of estimated combination coefficients within a client by a constant $s$ to avoid extreme division, which can be easily controlled in practice by adding a smoother while conducting data-division as in \cite{ref:fedsoft}. Lemma \ref{lemma:Perr} shows that the error rate of data division for any sample is bounded. It's proof is a modification of \cite[Lemma 3]{ref:ifca}. Following the same way as in \cite[Theorem 2]{ref:ifca}, we prove the convergence of the \textit{Data Division} stage.
\begin{assumption}
\label{ast:lip}
(Lipschitz-smoothness) $\forall \theta, \theta', f(\theta)\leq f(\theta') + \langle\nabla f(\theta'), \theta'-\theta\rangle + \frac{L}{2}\|\theta' - \theta\|^2$
\end{assumption}
\begin{assumption}
\label{ast:convex}
(Strong-Convexity) $\forall \theta, \theta', f(\theta)\geq f(\theta') + \langle\nabla f(\theta'), \theta'-\theta\rangle + \frac{\mu}{2}\|\theta' - \theta\|^2$
\end{assumption}
\begin{assumption}
\label{ast:delta}
(Stable-Initialize) $\theta_j^* = \operatorname*{arg\,min}_{\theta}{l(q_j;\theta)}$ is the theoretically optimum parameter for inherent distribution $q_j$, and $\Delta = \min_{j\neq j'}\|\theta_{j}^* - \theta_{j'}^*\|$, Assume that the initial VAE parameter for $q_j$ satisfies $\|\theta_{j}^{(0)}-\theta_j^*\| \leq (\frac{1}{2}-\alpha)\sqrt{\mu/L}\Delta$
\end{assumption}
\begin{assumption}
\label{ast:smooth}
$\forall i=0,\dots,n-1, j=0,\dots,m-1, j'\neq j$ The estimated prior probability $\alpha_i^j / \alpha_i^{j'} <= (1-s)/s$
\end{assumption}
\begin{lemma}
\label{lemma:Perr}
$\forall x\in D_i$, probability of $x$ incorrectly assigned to other distributions has upper bound:
\begin{equation}
    {P_{\text{err}}(x) \leq \frac{4m\eta^2}{s^2\alpha^2\mu^2\Delta^4}}
\end{equation}
\end{lemma}
\begin{proof}
Without loss of generality, we suppose sample $x$ from client $i$ belongs to $q_0$, and discuss the probability of $x$ assigned to $q_j$, denoted $P_{i}^{j}$. $\mathcal{L}^{0}(\theta)$ denotes the average loss of $x\sim q_0$ calculated by parameter $\theta$
\begin{equation}
    P_{i}^{j} = P\left\{l(x, j) \leq l(x, 0)\right\} = P\left\{\alpha^j\mathcal{L}(x;\theta_j^{(t)})\leq\alpha^0\mathcal{L}(x;\theta_0^{(t)})\right\}
\end{equation}
Using the probability inequality $P(a>b) \leq P(a>t)+P(b<t), \forall t$ and let $t = \frac{\alpha^0}{\alpha^0+\alpha^j}\left(\mathcal{L}^{0}(\theta_0^{(t)}) +\mathcal{L}^{0}(\theta_j^{(t)})\right)$. Then we have
\begin{equation}
    \begin{aligned}
        &P\left\{\alpha^0\mathcal{L}(x;\theta_0^{(t)})\geq t\right\} \\
        = &P\left\{\alpha^0\mathcal{L}(x;\theta_0^{(t)})\geq \frac{\alpha^0}{\alpha^0+\alpha^j}\left(\mathcal{L}^{0}(\theta_0^{(t)}) +\mathcal{L}^{0}(\theta_j^{(t)})\right)\right\}\\
        = &P\left\{\mathcal{L}(x;\theta_0^{(t)}) - \mathcal{L}^{0}(\theta_0^{(t)}) \geq \frac{\alpha^j}{\alpha^0+\alpha^j}\left(\mathcal{L}^{0}(\theta_j^{(t)})-\mathcal{L}^{0}(\theta_0^{(t)})\right)\right\} \\
        \leq &P\left\{\mathcal{L}(x;\theta_0^{(t)}) - \mathcal{L}^{0}(\theta_0^{(t)}) \geq s\left(\mathcal{L}^{0}(\theta_j^{(t)})-\mathcal{L}^{0}(\theta_0^{(t)})\right)\right\}
    \end{aligned}
\end{equation}

Based on Assumption\ref{ast:convex}, \ref{ast:delta} we have
\begin{equation}
    \begin{aligned}
        \mathcal{L}^{0}(\theta_j^{(t)}) &\geq \mathcal{L}^{0}(\theta_0^*) + \frac{\mu}{2}\|\theta_j^{(t)} - \theta_0^*\|^2 \\
        &\geq \mathcal{L}^{0}(\theta_0^*) + \frac{\mu}{2}\left(\|\theta_j^* - \theta_0^*\| - \|\theta_j^{(t)} - \theta_j^*\|\right)^2 \\
        &\geq \mathcal{L}^{0}(\theta_0^*) + \frac{\mu}{2}\left(\frac{1}{2}+\alpha\right)^2\Delta^2
    \end{aligned}
\end{equation}

Similarly, according to Assumption\ref{ast:lip}, \ref{ast:delta}
\begin{equation}
    \begin{aligned}
        \mathcal{L}^{0}(\theta_0^{(t)}) &\leq \mathcal{L}^{0}(\theta_0^*) + \frac{L}{2}\|\theta_0^{(t)} - \theta_0^*\|^2 \\
        &\leq \mathcal{L}^{0}(\theta_0^*) + \frac{L}{2}\left(\frac{1}{2}-\alpha\right)^2\frac{\mu}{L}\Delta^2 \\
        &= \mathcal{L}^{0}(\theta_0^*) + \frac{\mu}{2}\left(\frac{1}{2}-\alpha\right)^2\Delta^2
    \end{aligned}
\end{equation}

Then we get $\mathcal{L}^{0}(\theta_j^{(t)}) - \mathcal{L}^{0}(\theta_0^{(t)}) \geq \alpha\mu\Delta^2$. With Chebyshev's inequality, we have $P\left\{\alpha^0\mathcal{L}(x;\theta_0^{(t)})\geq t\right\} \leq \frac{4\eta^2}{s^2\alpha^2\mu^2\Delta^4}$ and similarly $P\left\{\alpha^0\mathcal{L}(x;\theta_j^{(t)})\leq t\right\} \leq \frac{4\eta^2}{s^2\alpha^2\mu^2\Delta^4}$. Using the probability inequality $P(a>b) \leq P(a>t)+P(b<t), \forall t$ we get an upper bound for $P_{i}^{j}$. Then we complete the proof according to $P_{\text{err}}(x) \leq \sum_{j\neq 0}P_i^j$.
\end{proof}
\begin{theorem}
(Convergence of Data Division) $\forall \delta\in(0,1)$,
\begin{equation}
\label{eq:conv}
    \begin{aligned}
        \|\theta_j^{(t+1)} - \theta_j^{(t)}\| \leq (1-\frac{p\mu}{8L})\|\theta_j^{(t)} - \theta_j^{*}\| + \frac{C_0}{\delta}\\ + \frac{C_1}{\delta^{3/2}} + C_2\frac{m\eta^2}{\delta s^2\alpha^2\mu^2\Delta^4}
    \end{aligned}
\end{equation}
holds with probability at least $1-\delta$, where $C_0, C_1$ are with respect to $m, |D_i|, n, \nu, \eta, \Delta, \alpha, \mu, L$\footnote{The form of $C_0, C_1$ can be seen in \cite[Theorem 2]{ref:ifca}, with $n'$ / total client number replaced by 1 / total data size}.
\end{theorem}

After the convergence of \textit{Data Division}, the result of the division can be seen as a constraint to the heterogeneity among local empirical distributions represented by the subsets. Here as discussed in \cite{ref:avgconv}, constant $\Gamma = L^* - \sum_{i=0}^{n-1}\alpha_{i}L_{i}^*$ is the measurement of heterogeneity. Here we use the convergence above to bound $\Gamma$. After $T$ rounds, where $T$ is large enough, $\|\theta_j^{(T)}-\theta_j^*\|$ can be ignored according to the convergence, then we have inequality \eqref{eq:gamma}, where we partition the local distribution into correctly-assigned samples(\textit{corr}) and error samples(\textit{err}), and then the average on \textit{corr} sample equals the average on $q_j$. Using Lemma\ref{lemma:Perr} and assuming that all inherent distributions has uniform prior probability and the expectation of $L(\cdot)$ is bounded by $\bar{\eta}$, we obtain an upper-bound of $\Gamma_j$, which ensures the convergence of \textit{Federated Training} according to \cite[Theorem 1]{ref:avgconv}.
\begin{figure}[htbp]
    \centering
    \begin{equation}
    \label{eq:gamma}
        \begin{aligned}
            \Gamma_j &= L^* - \sum_{i=0}^{n-1}\alpha_{i}L_{i}^* = \sum_{i=0}^{n-1}\alpha_{i}\left(L^*-L_i^*\right) \\
            &= \sum_{i=0}^{n-1}\frac{|\mathcal{S}_i^j|}{\sum_{i}|\mathcal{S}_i^j|} \min \mathbb{E}_{q_j}[L(x;\omega_j)]\\
            &\quad - \min\left(\mathbb{E}_{corr}^i[L(x;\omega_j)]+\mathbb{E}_{err}^i[L(x;\omega_j)]\right)\\
            &\leq \sum_{i=0}^{n-1}\frac{|\mathcal{S}_i^j|}{\sum_{i}|\mathcal{S}_i^j|} \min \mathbb{E}_{q_j}[L(x;\omega_j)] - \min \mathbb{E}_{corr}^i[L(x;\omega_j)]\\
            &\quad - \min \mathbb{E}_{err}^i[L(x;\omega_j)]\\
            &= \frac{1}{\sum_{i}|\mathcal{S}_i^j|}\sum_{i=0}^{n-1}\mathbb{E}[n_{err}]\left(\mathbb{E}_{q_j}[L(x;\omega_j)] - \mathbb{E}_{\bar{q_j}}[L(x;\omega_j)]\right)\\
            &\leq c_1\frac{m\eta^2}{s^2\alpha^2\mu^2\Delta^4}2\bar{\eta}
        \end{aligned}
    \end{equation}
\end{figure}

Combining the conclusion of \textit{Data Division} and \textit{Federated Training}, the global convergence of FedGMI is established.

\begin{remark}
    \label{rmk:M}
    The theoretical analysis above reveals that the stability of local data partitioning is critically dependent on \(\Delta = \min_{j\neq j'}\|\theta_{j}^* - \theta_{j'}^*\|\), which represents the discriminability of the inherent distributions. As \(M\) increases, the optimum parameters for each inherent distribution tend to be more congested in the parameter space. This leads to a reduction in $\Delta$, thereby increasing the partitioning error rate $P_{err}$ and potentially compromising the stability of the mixture inference.
\end{remark}

\section{Experimental Analysis}
\label{sec:EX}
\subsection{Settings}
The proposed framework, FedGMI, is evaluated on two standard image classification benchmarks: MNIST and CIFAR-10. To simulate heterogeneous client environments, we generate local datasets following the Probabilistic Mixture (PM) patterns described in \cite{ref:fedsoft, ref:fedce}.
For both datasets, the inherent distributions that form the \textbf{PM} components are generated via image rotation. Specifically, for the M=2 scenario, the mixture consists of the original images and their \(90^{\circ}\) counterclockwise-rotated counterparts. For the M=3 scenario, we introduce an additional component of \(180^{\circ}\) rotated images. We define the distribution of original images as $p_0$, and the distributions of \(90^{\circ}\) and \(180^{\circ}\) counterclockwise rotated images as $p_1$ and $p_2$. No rotation-specific architectural or initialization bias is introduced to the VAEs. All neural networks in our experiment are randomly initialized. 

The analysis focuses on scenarios with $M=2$ and $M=3$ latent distributions. As noted in Remark \ref{rmk:M}, higher values of $M$ necessitate greater discriminability between inherent distributions to maintain partitioning stability. In addition, a low-dimensional inherent distribution mixture facilitates intuitive visualization and interpretation of the learned mixture structures, which is particularly useful for understanding how clients' local data distributions are composed.  Larger \(M\) also introduce significantly higher computational complexity. In this work, we therefore focus on small \(M\) to strike a balance between stability and interpretability, while leaving the exploration of large \(M\) situations to future work.

To further validate FedGMI under real-world heterogeneous conditions beyond synthetically constructed rotations, we extend our experiments to the EMNIST dataset. Unlike the synthetic rotations used in MNIST and CIFAR-10, heterogeneity in this case arises from the intrinsic morphological variations between uppercase and lowercase letters. In this setting, the two inherent distributions correspond to uppercase and lowercase letters respectively, providing a more realistic benchmark for the PM scenario.

We compare FedGMI against two strong baselines: IFCA\cite{ref:ifca} and FedSoft\cite{ref:fedsoft}. We modified the original FedSoft workflow since the original FedSoft selects users for each inherent distribution based on the estimated weights, which leads to an effective training size larger than \(k\) clients and non-uniform participation probabilities across clients. 
To maintain a controlled experimental environment, we standardized the client selection process across all methods, ensuring that each round involves a fixed number of participating clients with uniform selection probability.

We assess performance using three primary metrics:
\begin{enumerate}
    \renewcommand{\theenumi}{\alph{enumi}}
    \renewcommand{\labelenumi}{\theenumi)}
    \item \textbf{Accuracy of Distribution-Expert Classifiers}:  This measures the accuracy of distribution-expert classifiers learned for each inherent distributions, reflecting FedGMI's ability to effectively capture the pre-set inherent distributions.
    \item \textbf{Client-associated accuracy}: This evaluates the accuracy on each client's local dataset, where samples are partitioned and classified by their corresponding expert classifiers. This metric evaluates the personalization performance of FedGMI.
    \item \textbf{Mixture factors}: We compare the inferred mixture factors $\alpha_i^j$ against the ground truth to evaluate the precision of the mixture inference.
\end{enumerate}

To evaluate FedGMI under different communication constraints, we consider two selection sizes $K$ per communication round to simulate different communication case: a low communication case ($K=5$) and a high communication case ($K=20$). This allows us to assess the framework's stability and convergence across varying levels of client participation.

\subsection{Implementation Details}
Our proposed framework, FedGMI, and all baseline methods were implemented using the \textit{PyTorch 2.4.0} framework with \textit{CUDA 12.2}. All experiments were conducted on a server equipped with four \textit{NVIDIA RTX 4090} GPUs, each with 24GB of VRAM, and \textit{Ubuntu 20.04.4} as the operating system. For all experiments, we set the number of clients $N=50$ and the maximum communication rounds $T_{max} = 200$ for MNIST and $T_{max} = 500$ for CIFAR-10. We referred to \cite{ref:trainvae} and added normalization terms to the VAE training objective to prevent posterior corruption. We adopt MLP as the backbone structure of VAEs and classifiers for MNIST and CNN for CIFAR-10 and EMNIST. For a more detailed demonstration of hyperparameters, network structures, and to ensure full reproducibility, please refer to our open-source code repository, which will be publicly available upon publication.

\begin{table*}[htbp]
\centering
\resizebox{0.8\textwidth}{!}{%
\begin{tabular}{@{}llccccccccc@{}}
\toprule
 &  & \multicolumn{3}{c}{IFCA} & \multicolumn{3}{c}{FedSoft} & \multicolumn{3}{c}{FedGMI} \\ \midrule
\multicolumn{1}{l|}{} &  & \multicolumn{1}{l}{$c_0$} & \multicolumn{1}{l}{$c_1$} & \multicolumn{1}{l}{client} & \multicolumn{1}{l}{$c_0$} & \multicolumn{1}{l}{$c_1$} & \multicolumn{1}{l}{client} & \multicolumn{1}{l}{$c_0$} & \multicolumn{1}{l}{$c_1$} & \multicolumn{1}{l}{client} \\ \cmidrule(l){3-11} 
\multicolumn{1}{l|}{MNIST-5} & $p_0$ & {\ul 85.65} & 4.05 & 90.03 & 90.72 & {\ul 90.76} & \textbf{95.78} & {\ul \textbf{93.22}} & 68.98 & 95.00 \\
\multicolumn{1}{l|}{} & $p_1$ & {\ul 89.25} & 3.59 &  & {\ul 90.58} & 90.35 &  & 31.16 & {\ul \textbf{93.75}} &  \\ \midrule
\multicolumn{1}{l|}{MNIST-20} & $p_0$ & 91.28 & {\ul 93.86} & \textbf{97.8} & 91.76 & {\ul \textbf{92.38}} & 96.82 & {\ul 91.98} & 48.85 & 94.89 \\
\multicolumn{1}{l|}{} & $p_1$ & {\ul \textbf{94.04}} & 89.13 &  & {\ul 92.08} & 91.36 &  & 64.44 & {\ul 93.79} &  \\ \midrule
\multicolumn{1}{l|}{CIFAR10-5} & $p_0$ & {\ul 66.79} & 10.00 & 74.68 & {\ul 56.11} & 53.51 & \textbf{98.98} & {\ul \textbf{67.86}} & 57.66 & 82.70 \\
\multicolumn{1}{l|}{} & $p_1$ & {\ul 67.61} & 10.00 &  & 53.76 & {\ul 56.77} &  & 57.46 & {\ul \textbf{68.23}} &  \\ \midrule
\multicolumn{1}{l|}{CIFAR10-20} & $p_0$ & {\ul 64.96} & 60.58 & 83.35 & 56.05 & {\ul 61.66} & \textbf{99.59} & {\ul \textbf{67.04}} & 60.22 & 85.30 \\
\multicolumn{1}{l|}{} & $p_1$ & {\ul 67.04} & 58.64 &  & {\ul 60.15} & 58.64 &  & 55.41 & {\ul \textbf{67.52}} &  \\ \bottomrule
\end{tabular}%
}
\caption{\textbf{Accuracy of classification} ($M=2$): a cross-evaluation to characterize the $M$ models automatically discovered by each method. The columns, $c_{j}$, denote these learned models, whose indexing is arbitrary. The rows, $p_{j}$, represent the test datasets of distinct distributions. Each cell shows the performance of a given model on a specific test distribution, revealing its learned specialization. The title of each row follows the format \textit{Dataset-Selection Size}, e.g., MNIST-5.}
\label{tab:m_2}
\end{table*}

\subsection{Evaluation under Synthetic Heterogeneity}
\subsubsection{$M=2$ scenario}

The first experiment is the $M=2$ scenario, where only 2 inherent distributions exist. This makes the mixture structure clear and intuitive so that we can infer the distribution VAEs actually learned by observing the proportions of local data division. 

\begin{figure}[h!]
\centering
\begin{tabular}{cc} 
    \subcaptionbox{MNIST, selection size = 5 \label{fig:alpha(a)}}{
        \includegraphics[width=0.42\linewidth]{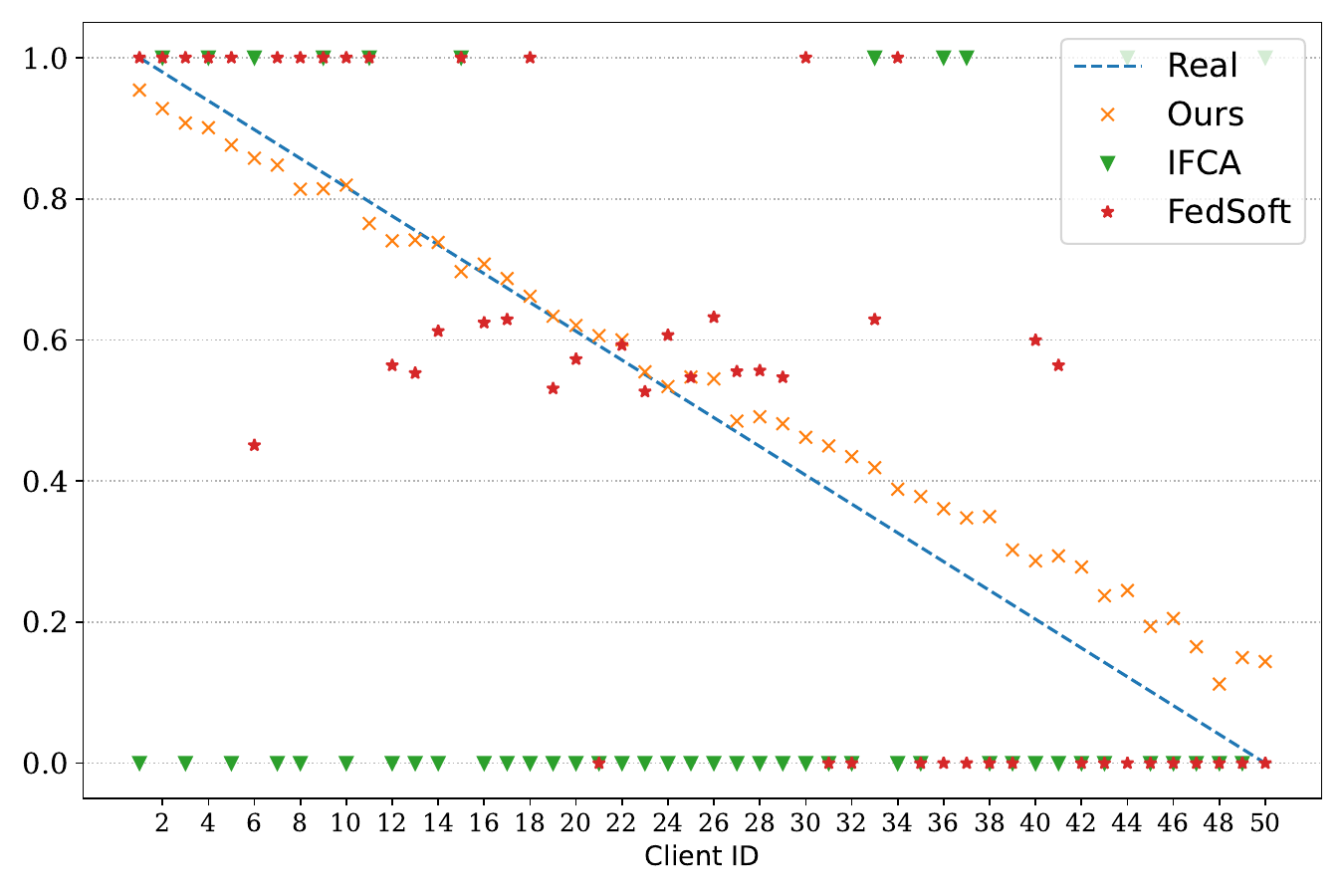}
    } &
    \subcaptionbox{MNIST, selection size = 20 \label{fig:alpha(b)}}{
        \includegraphics[width=0.42\linewidth]{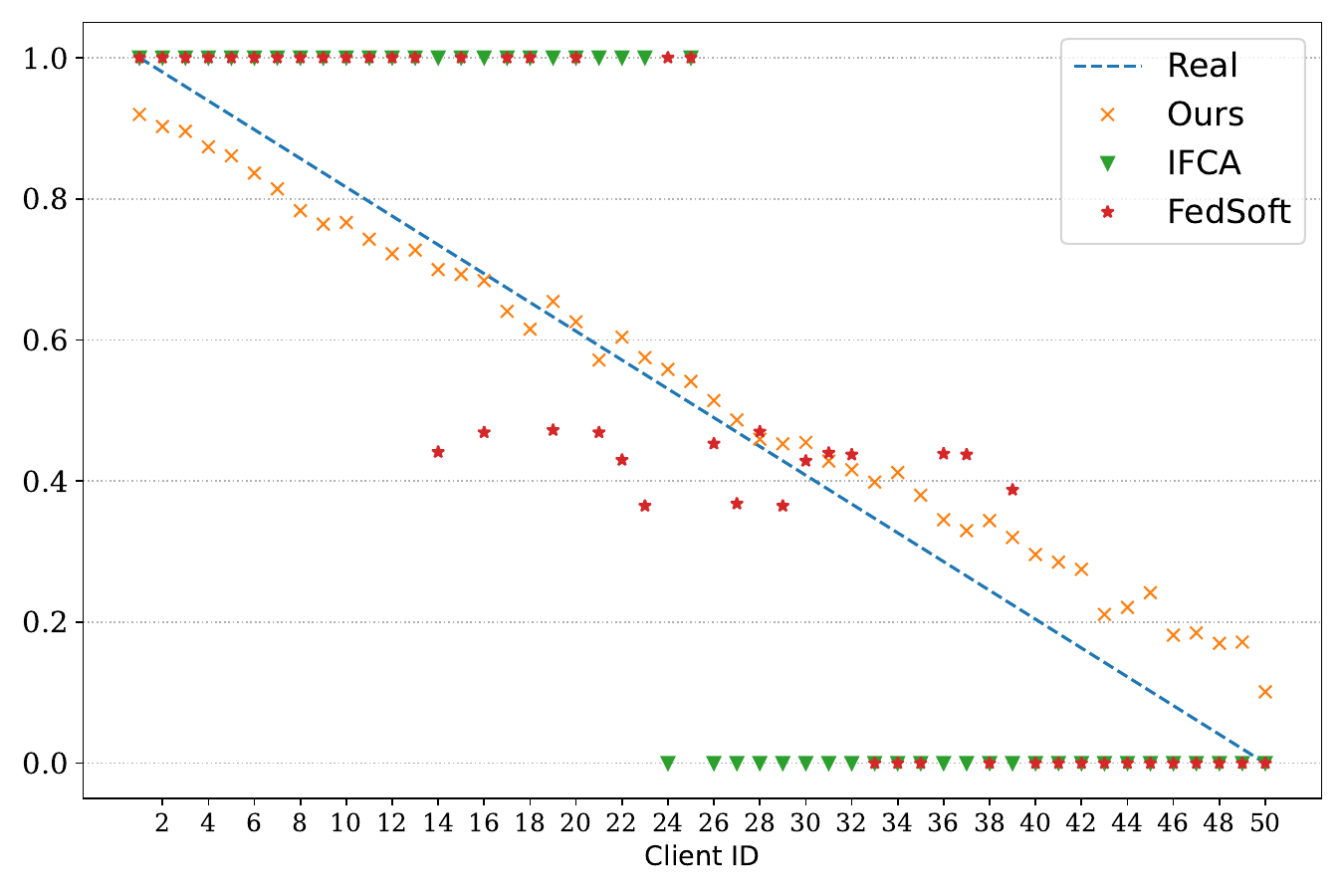}
    } \\
    \\
    \subcaptionbox{CIFAR10, selection size = 5 \label{fig:alpha(c)}}{
        \includegraphics[width=0.42\linewidth]{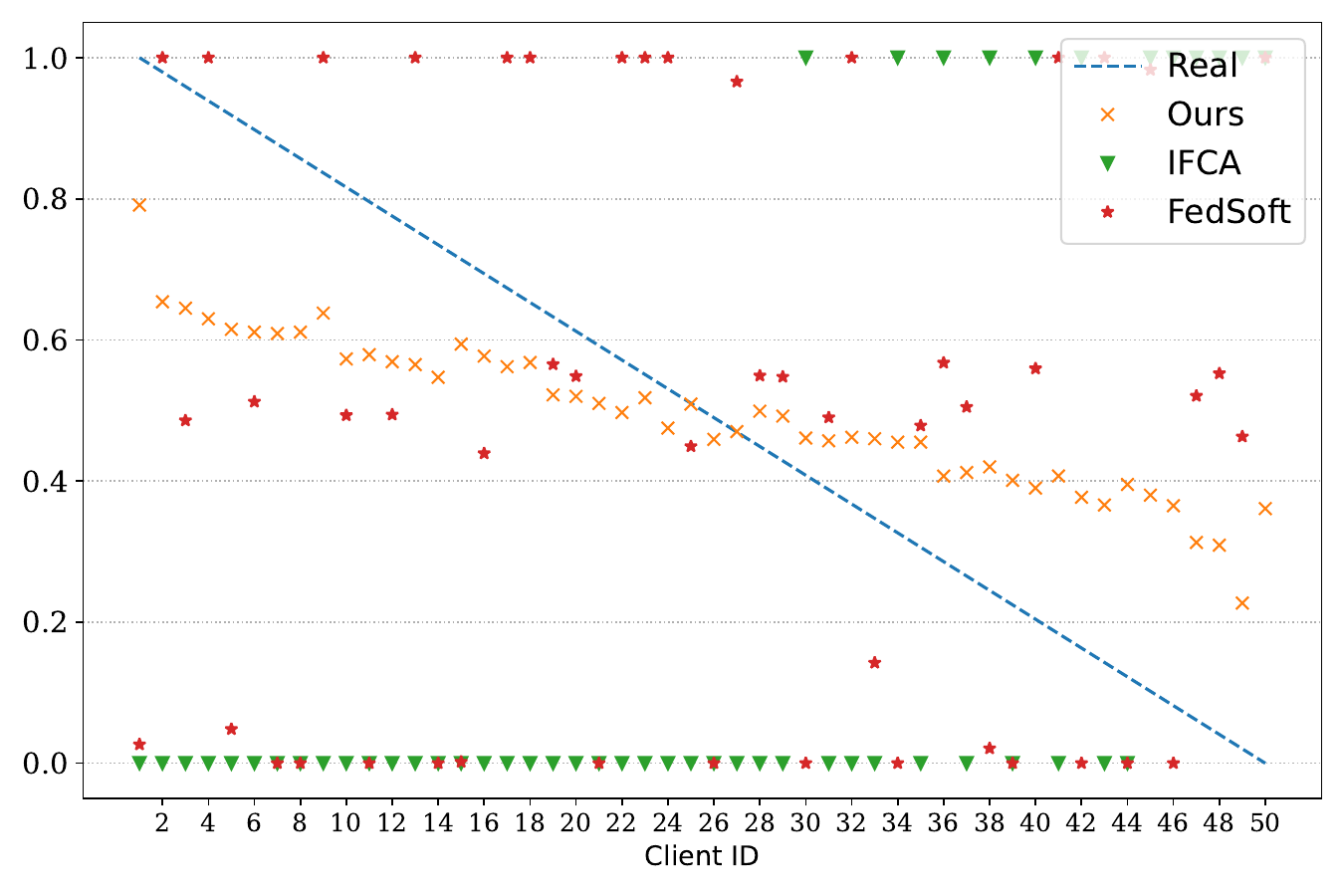}
    } &
    \subcaptionbox{CIFAR10, selection size = 20 \label{fig:alpha(d)}}{
        \includegraphics[width=0.42\linewidth]{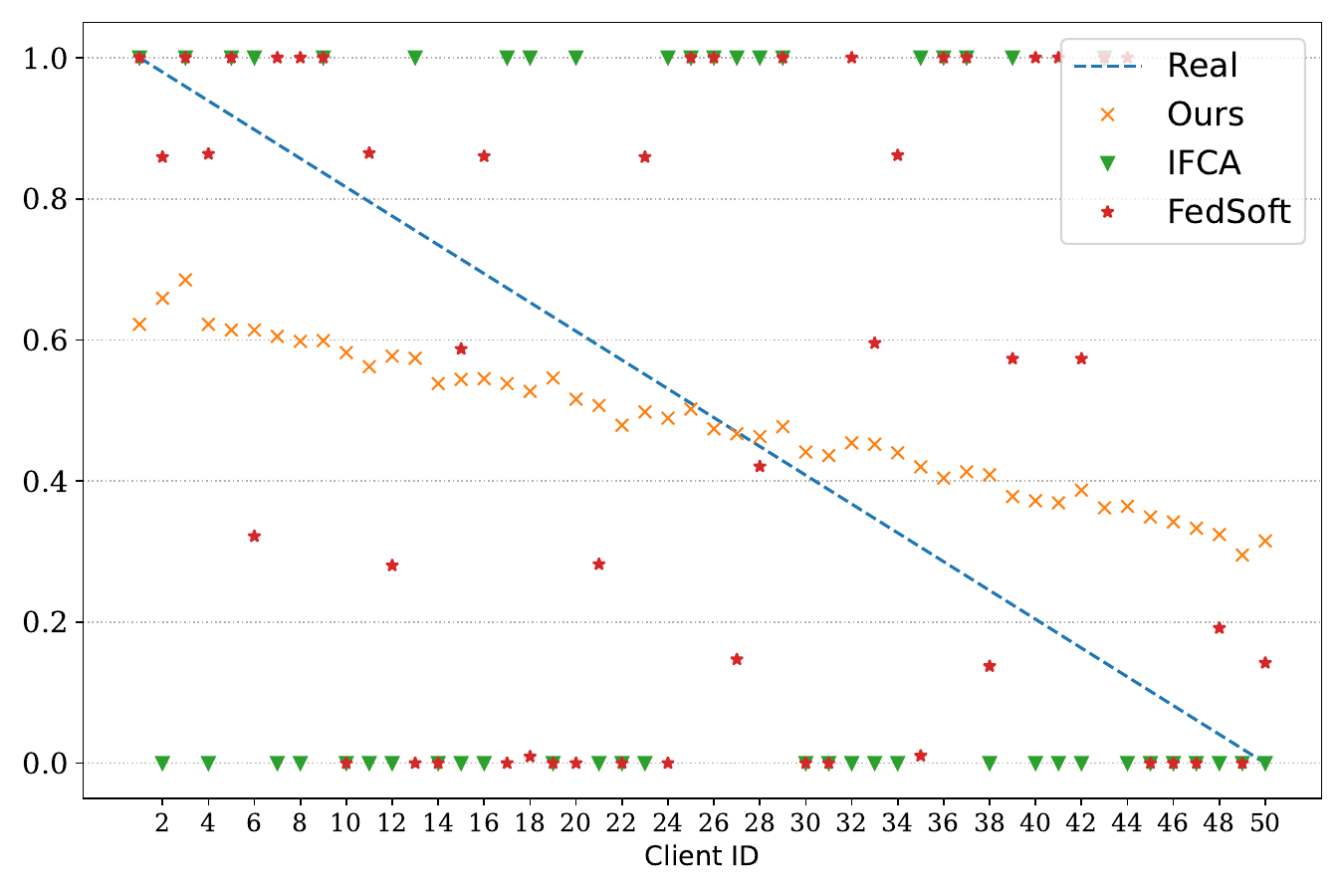}
    }
\end{tabular}
\caption{Estimation of the proportion of $p_0$ in $M=2$ situation.}
\label{fig:alpha}
\end{figure}
Fig.\ref{fig:alpha} illustrates the estimated proportion of samples from $p_0$ in each client’s local dataset under the setting where the local data distributions are artificially constructed to vary linearly across clients. The blue dashed line represents the ground-truth proportions, while the scatter points denote the estimated proportions produced by each method, whose values range from 0 to 1. 
On the relatively simple MNIST dataset, FedGMI provides accurate and robust estimation of the distribution proportions across both communication settings. For the more complex CIFAR-10 dataset, although the predicted proportions exhibit a compressed dynamic range compared to the ground truth, they strictly preserve the underlying monotonic trend and remain stable under varying communication costs. This phenomenon suggests that FedGMI tends to produce more conservative estimates as data complexity increases while preserving the monotonic trend.
IFCA does not explicitly output mixture proportions but instead assigns each client to a single cluster, and thus its outputs can only be interpreted as binary proportion estimates. Through this perspective, we can still examine the algorithm’s ability to perceive the relationships among clients. IFCA separates clients into two clusters reasonably well on MNIST under the high-communication setting (Fig.\ref{fig:alpha(b)}), while the assignments deviate significantly from the ground truth in the low communication case and on CIFAR-10.
FedSoft exhibits a stage-like pattern on MNIST that remains stable across settings but lacks the precision and granularity compared to FedGMI. On CIFAR-10, FedSoft still shows a stage-like pattern but with significantly increased noise.

The aforementioned linear pattern of distributions represents an idealized scenario, whereas in practical settings, user distributions are typically random. Table.\ref{tab:m_2} presents the classification accuracy for clients with randomly mixed distributions, where the ground truth proportion of $p_0$ is sampled from $Uniform[0,1]$. 

In the low communication case, IFCA converges to a single model for both datasets, with the second model failing to provide valid classification. 
Both FedSoft and FedGMI successfully yield two distinct classifiers specialized for different distributions. In this setting, FedGMI exhibits a more pronounced performance gap when each learned classifier is evaluated on two different distributions. Specifically, each classifier maintains high accuracy on its target distribution while showing significantly lower accuracy on mismatched ones. 
FedSoft also achieves high peak accuracy; however, the gap between its performance on matched and mismatched test sets is smaller compared to FedGMI, particularly on the CIFAR-10 dataset.

In the high communication case, FedGMI achieves competitive results in accuracy and maintains a consistent deviation between clusters for both distributions across both datasets. IFCA shows the ability to distinguish clusters on MNIST, while the result is unbalanced on CIFAR-10. FedSoft performs slightly better than the low communication case. FedGMI achieves slightly lower accuracy on the best-matched inherent distributions of each classifier compared to the low communication case, but outperforms the mismatched ones, suggesting a potential misallocation of computational resources. The optimum communication consumption for FedGMI needs to be studied in future research.

Fig.\ref{fig:acccurve} reports the classification accuracy curves of different methods on the CIFAR-10 dataset under two communication settings. Due to space limitations, we only present the CIFAR-10 results, as CIFAR-10 constitutes a more challenging and representative benchmark compared with MNIST. Overall, FedGMI demonstrates consistently strong performance across all settings. More importantly, a clear and stable performance gap can be observed for the same model evaluated on different inherent distributions throughout the training process. In contrast, FedSoft shows much smaller inter-distribution gaps, particularly under the low communication case, and its accuracy increases more slowly. IFCA, while achieving relatively fast convergence and competitive accuracy, tends to concentrate training resources on a single effective model (Fig.\ref{fig:ifca_cifar5}) or to learn models without meaningful inter-distribution differentiation (Fig.\ref{fig:ifca_cifar20}) , which leads to the collapse of mixture inference. The trends observed in Fig.\ref{fig:acccurve} are consistent with the results reported in Fig.\ref{fig:alpha} and Table.\ref{tab:m_2}.

Fig.\ref{fig:sample} shows an example of the generated samples of each inherent VAE to intuitively observe the learned inherent distributions. As shown in Fig.\ref{fig:sample(a)} and \ref{fig:sample(b)}, $p_0$ generates handwritten digits with original angle and $p_1$ generates digits rotated $90^{\circ}$ counter-clockwise, which is consistent with the way we construct the two distributions. These samples suggest that our method can capture heterogeneity among multiple data distributions in a way that aligns with human perception.

From the experimental results above, we observe that FedGMI demonstrates a superior capacity to capture the pre-set rotation-based heterogeneity. This is evidenced by several key factors: the estimated mixture proportions are more consistent with the ground truth, the expert classifiers exhibit a more obvious performance gap across different distributions, and the final classification accuracy leads the baselines in most scenarios.
In contrast, IFCA and FedSoft show limited sensitivity to rotation-based heterogeneities. This may suggest that these frameworks are capturing alternative latent structural combinations rather than implying an lack of capability. 
However, within our controlled experimental environment where class distributions are uniform across clients, rotation constitutes the predominant source of heterogeneity. Consequently, the results indicate that FedGMI captures the dominant factor of heterogeneity of the data distribution compared to the baselines. Furthermore, the stability of this performance across varying $K$ validates that FedGMI is robust under diverse communication constraints.

\begin{figure}[h]
    \centering
    \setlength{\tabcolsep}{4pt}
    \renewcommand{\arraystretch}{1.15}

    \begin{tabular}{cc}
        \subcaptionbox{FedGMI--CIFAR10--5\label{fig:gmi_cifar5}}{
            \includegraphics[width=0.45\columnwidth]{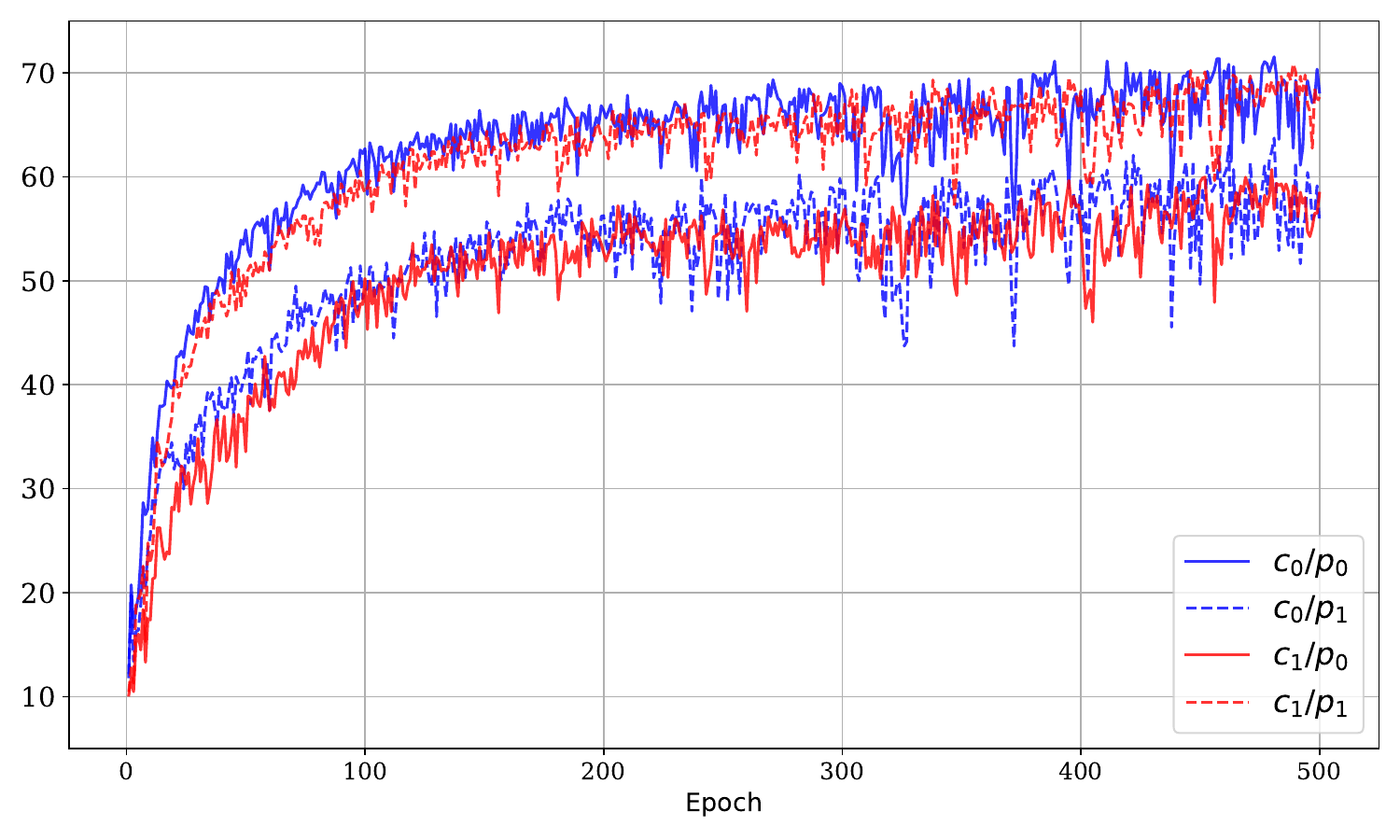}
        } &
        \subcaptionbox{FedGMI--CIFAR10--20\label{fig:gmi_cifar20}}{
            \includegraphics[width=0.45\columnwidth]{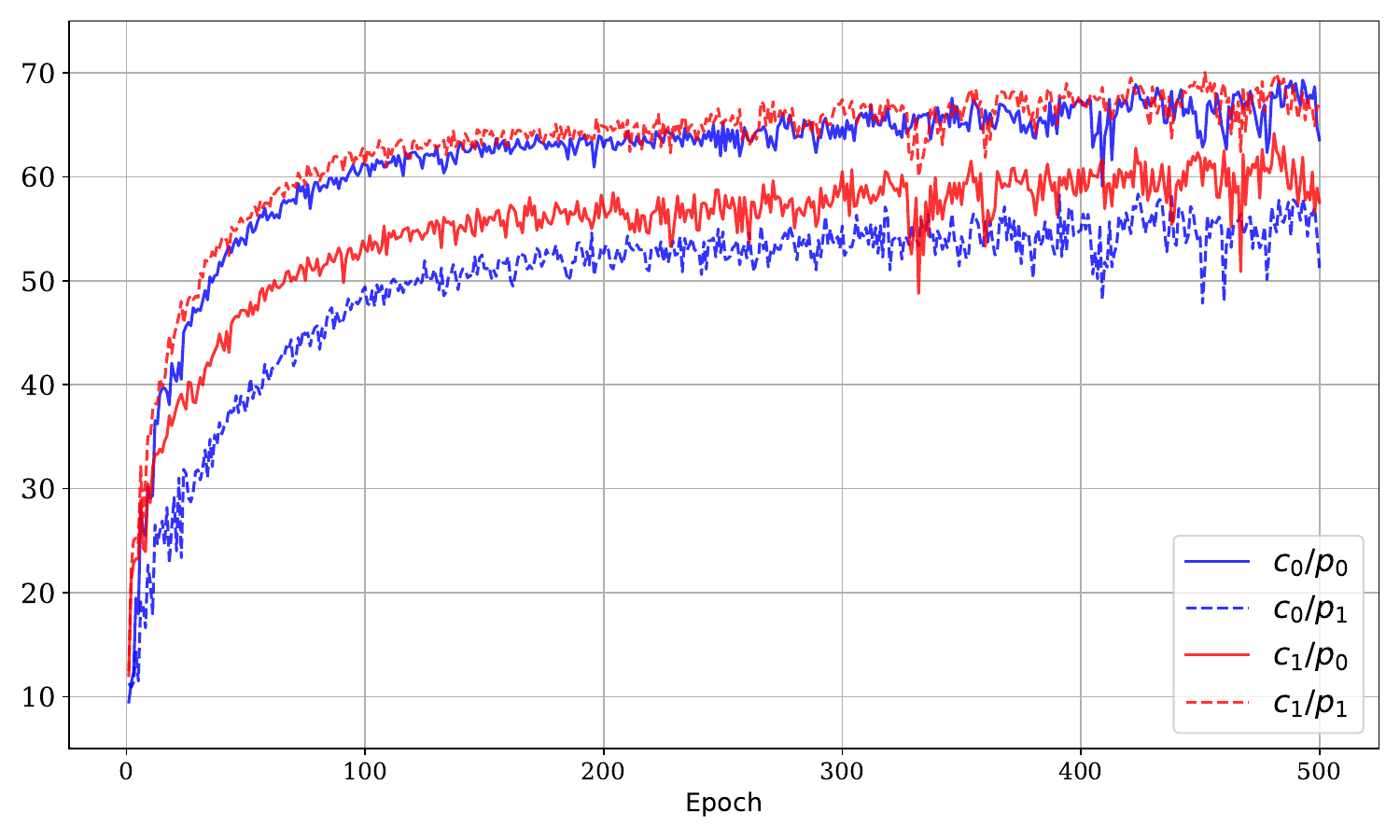}
        } \\

        \subcaptionbox{FedSoft--CIFAR10--5\label{fig:soft_cifar5}}{
            \includegraphics[width=0.45\columnwidth]{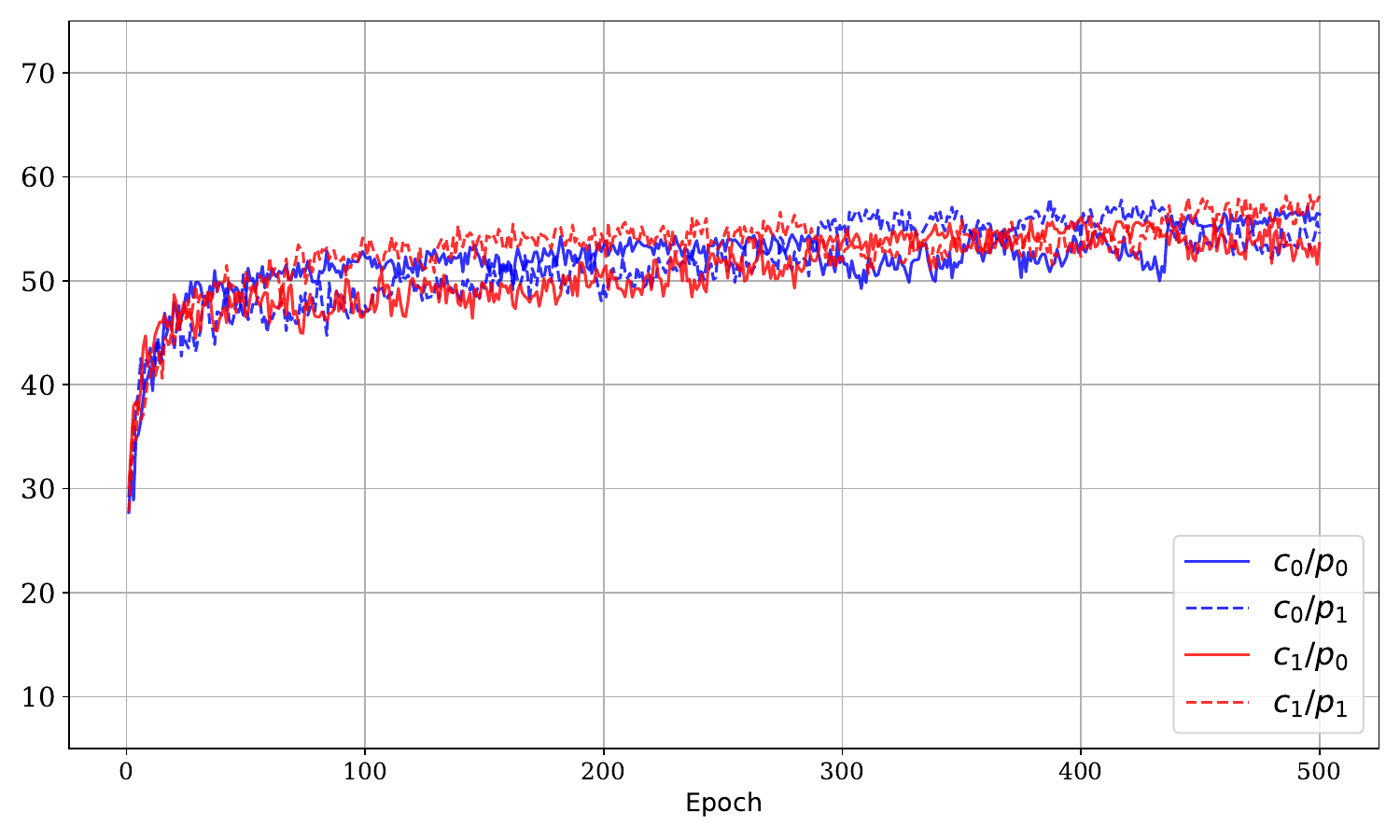}
        } &
        \subcaptionbox{FedSoft--CIFAR10--20\label{fig:soft_cifar20}}{
            \includegraphics[width=0.45\columnwidth]{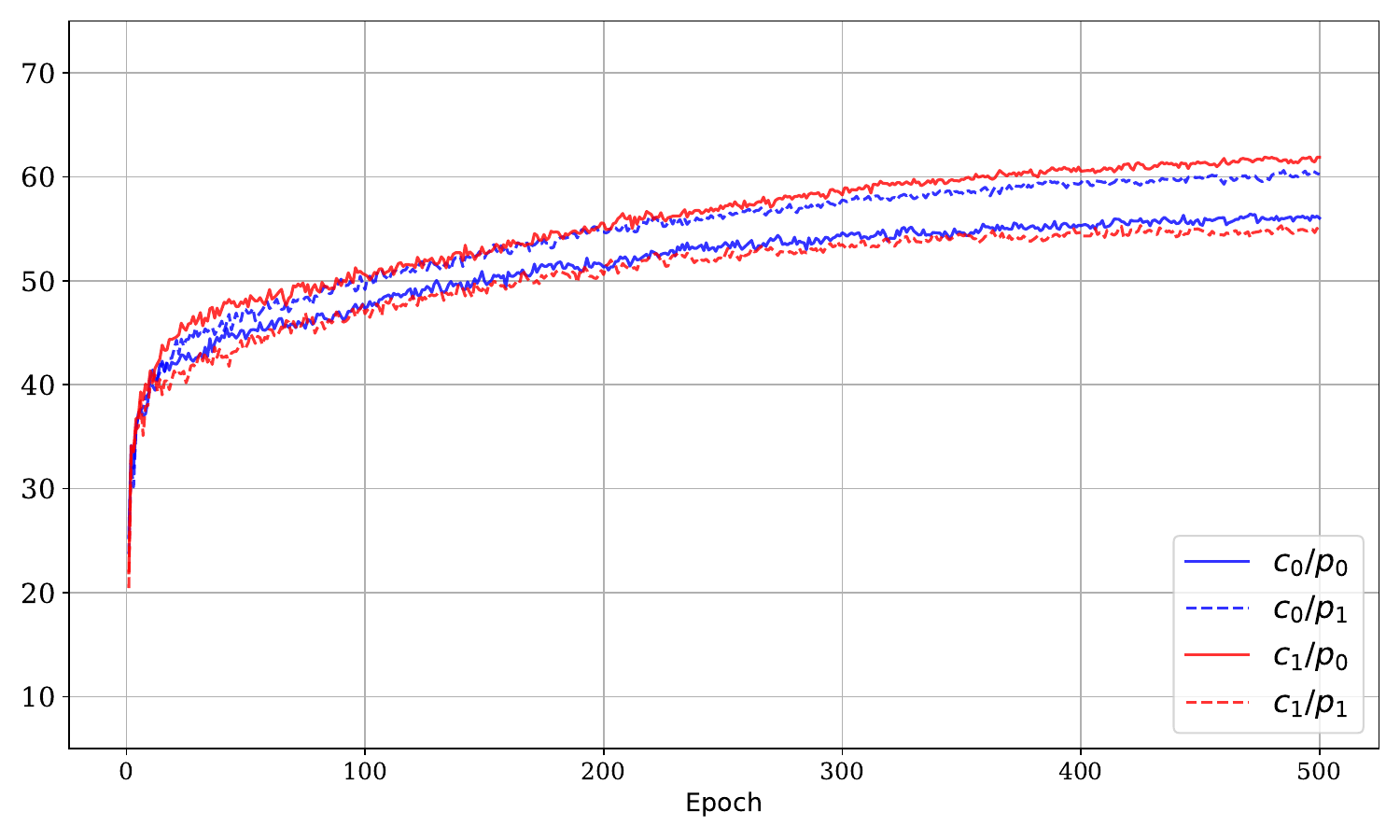}
        } \\

        \subcaptionbox{IFCA--CIFAR10--5\label{fig:ifca_cifar5}}{
            \includegraphics[width=0.45\columnwidth]{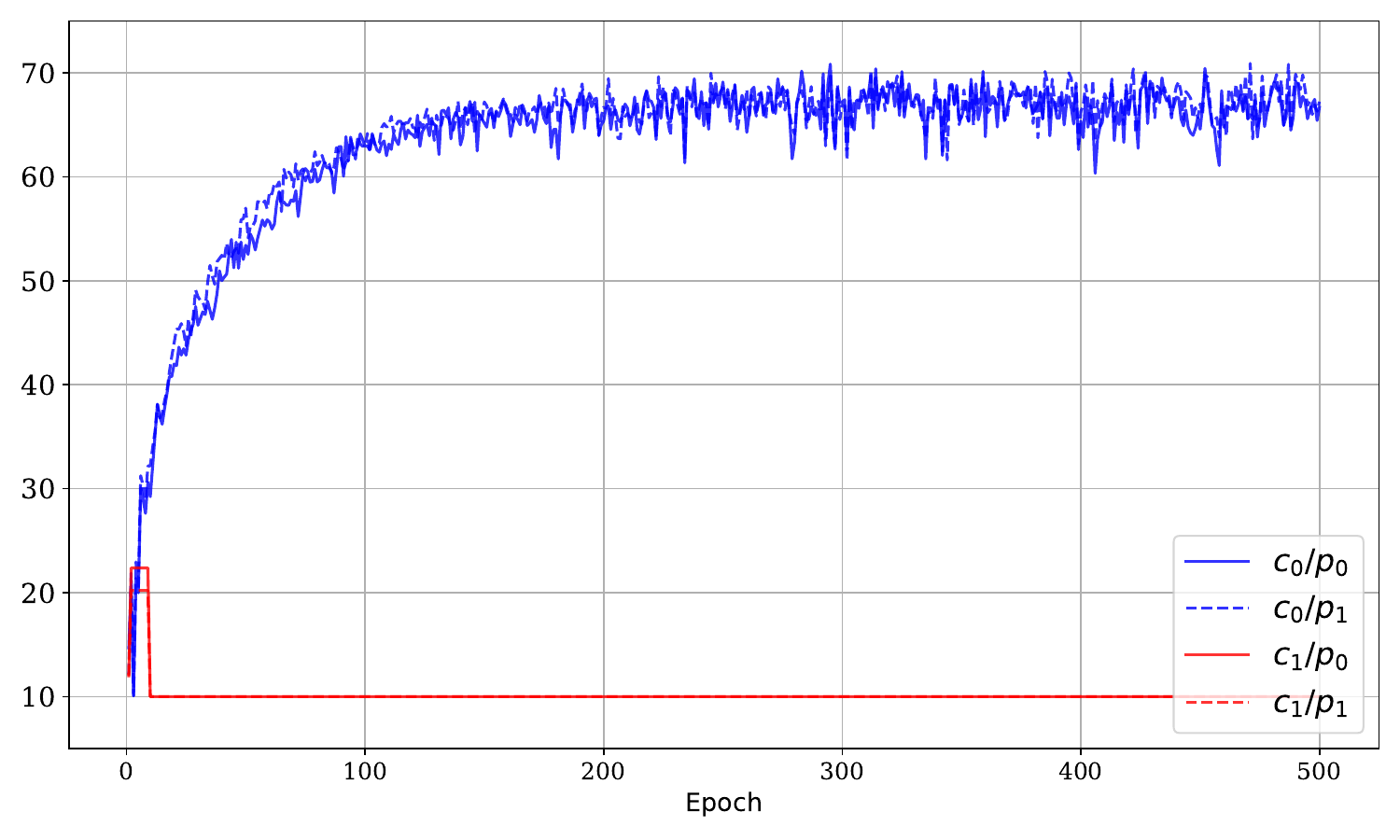}
        } &
        \subcaptionbox{IFCA--CIFAR10--20\label{fig:ifca_cifar20}}{
            \includegraphics[width=0.45\columnwidth]{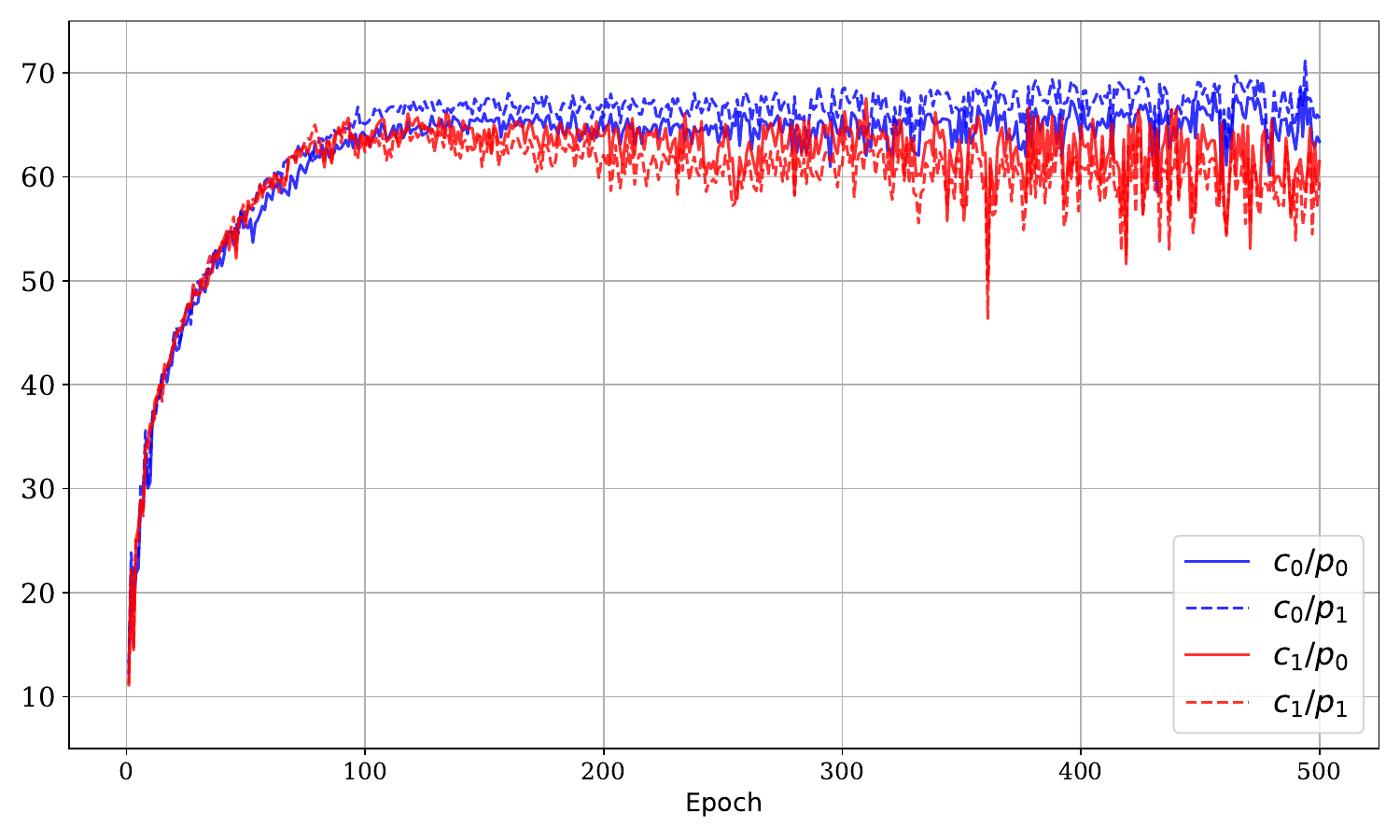}
        } \\
    \end{tabular}

    \caption{\textbf{Classification accuracy on CIFAR-10 over training rounds.}
    Each subfigure is labeled in the format \textit{Method-Dataset-Selection Size}. Different colors represent different clusters, and solid and dashed lines indicate the performance on different components of the  data distribution mixture. }
    \label{fig:acccurve}
\end{figure}

\begin{figure}[htb]
\centering
\subfloat[$p_0$]{
		\includegraphics[width=0.8\columnwidth]{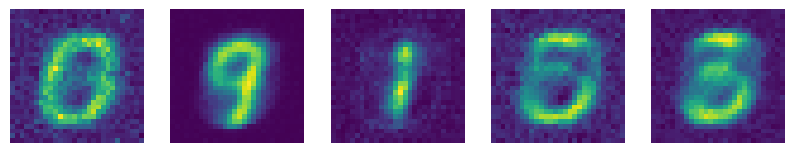}
		\label{fig:sample(a)}}\\
\subfloat[$p_1$]{
		\includegraphics[width=0.8\columnwidth]{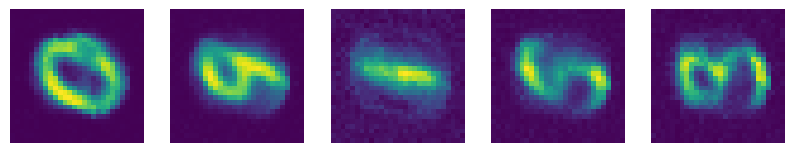}
		\label{fig:sample(b)}}
\caption{Image generated by the VAEs in $M=2$ situation}
\label{fig:sample}
\end{figure}

\subsubsection{$M=3$ scenario} The second experiment investigates the $M=3$ scenario. In this case, the combination fraction ${\alpha_i^j}$ is in 3-dimensional space, which cannot be visualized as intuitively as shown in Fig.\ref{fig:alpha}. 

\begin{table*}[htbp]
\centering
\resizebox{0.9\textwidth}{!}{%
\begin{tabular}{@{}cllllclllclllc@{}}
\toprule
\multicolumn{1}{l}{} &  & \multicolumn{4}{l}{IFCA} & \multicolumn{4}{l}{FedSoft} & \multicolumn{4}{l}{FedGMI} \\ \midrule
\multicolumn{1}{l|}{} &  & \multicolumn{1}{c}{$c_0$} & \multicolumn{1}{c}{$c_1$} & \multicolumn{1}{c}{$c_2$} & client & \multicolumn{1}{c}{$c_0$} & \multicolumn{1}{c}{$c_1$} & \multicolumn{1}{c}{$c_2$} & client & \multicolumn{1}{c}{$c_0$} & \multicolumn{1}{c}{$c_1$} & \multicolumn{1}{c}{$c_2$} & client \\ \cmidrule(l){3-14} 
\multicolumn{1}{c|}{\multirow{3}{*}{MNIST-5}} & $p_0$ & {\ul 91.18} & 84.68 & 87.59 & \multirow{3}{*}{92.50} & 84.22 & 85.64 & 85.02 & \multirow{3}{*}{\textbf{93.11}} & 34.15 & 39.18 & {\ul \textbf{93.36}} & \multirow{3}{*}{91.77} \\
\multicolumn{1}{c|}{} & $p_1$ & 88.63 & {\ul 92.33} & 88.99 &  & 83.17 & 82.56 & 83.24 &  & {\ul \textbf{93.99}} & 28.47 & 23.66 &  \\
\multicolumn{1}{c|}{} & $p_2$ & 88.28 & 88.54 & {\ul 91.10} &  & {\ul 91.65} & {\ul 92.07} & {\ul 91.98} &  & 29.40 & {\ul \textbf{93.26}} & 56.28 &  \\ \midrule
\multicolumn{1}{c|}{\multirow{3}{*}{MNIST-20}} & $p_0$ & {\ul \textbf{91.58}} & 86.47 & 89.81 & \multirow{3}{*}{92.38} & 88.50 & 89.25 & 88.23 & \multirow{3}{*}{\textbf{96.23}} & 18.47 & 78.98 & {\ul 91.19} & \multirow{3}{*}{92.47} \\
\multicolumn{1}{c|}{} & $p_1$ & 88.66 & 89.59 & {\ul 90.00} &  & 87.10 & 87.22 & 88.15 &  & {\ul \textbf{93.80}} & 17.61 & 82.48 &  \\
\multicolumn{1}{c|}{} & $p_2$ & 86.11 & {\ul 90.53} & 88.75 &  & {\ul 92.69} & {\ul 92.41} & {\ul \textbf{92.49}} &  & 34.64 & {\ul 91.48} & 55.59 &  \\ \midrule
\multicolumn{1}{c|}{\multirow{3}{*}{CIFAR10-5}} & $p_0$ & 48.12 & {\ul \textbf{68.08}} & 51.50 & \multirow{3}{*}{74.78} & 54.17 & {\ul 55.82} & 47.29 & \multirow{3}{*}{\textbf{99.02}} & 59.03 & 64.06 & {\ul 66.76} & \multirow{3}{*}{72.86} \\
\multicolumn{1}{c|}{} & $p_1$ & 51.43 & {\ul 67.10} & 50.74 &  & {\ul 55.74} & 53.88 & 48.95 &  & {\ul \textbf{70.91}} & 56.00 & 54.20 &  \\
\multicolumn{1}{c|}{} & $p_2$ & 54.59 & {\ul 64.25} & 45.29 &  & 52.57 & {\ul 53.63} & 47.64 &  & 54.96 & {\ul \textbf{68.20}} & 59.99 &  \\ \midrule
\multicolumn{1}{c|}{\multirow{3}{*}{CIFAR10-20}} & $p_0$ & {\ul 63.89} & 62.55 & 61.65 & \multirow{3}{*}{74.78} & 53.56 & 55.12 & {\ul 56.33} & \multirow{3}{*}{\textbf{96.23}} & 57.60 & {\ul \textbf{66.46}} & 62.57 & \multirow{3}{*}{81.06} \\
\multicolumn{1}{c|}{} & $p_1$ & 63.41 & 60.43 & {\ul 64.40} &  & {\ul 53.09} & 51.19 & 50.00 &  & {\ul \textbf{66.76}} & 55.55 & 60.49 &  \\
\multicolumn{1}{c|}{} & $p_2$ & 63.29 & 61.18 & {\ul \textbf{65.65}} &  & 53.09 & {\ul 54.92} & 53.34 &  & 56.44 & 62.42 & {\ul 64.84} &  \\ \bottomrule
\end{tabular}%
}
\caption{\textbf{Accuracy of classification} ($M=3$): the notation for rows and columns is identical to that used in Table.\ref{tab:m_2}.}
\label{tab:m_3}
\end{table*}

Consistent with Table.\ref{tab:m_2}, Table.\ref {tab:m_3} shows the classification accuracy in the $M=3$ scenario under both low and high communication cases, where the local data distributions are random mixtures of the inherent distributions. 
Consistent with the trends observed in the $M=2$ case, FedGMI maintains distinct and specialized classifiers for each distribution across both datasets. In contrast, the performance profiles of FedSoft and IFCA show less differentiation between the learned models. The significant accuracy gap maintained by FedGMI when evaluating models across different distributions suggests a more precise partitioning of local data. Additionally, FedGMI achieves higher classification accuracy in the majority of tested scenarios compared to the baselines.

Experimental results of two scenarios indicate that FedGMI is a stable and high-efficiency framework for Federated Learning under probabilistic mixture situations. VAEs in our algorithm successfully divided the local data, which can be used to support the classification task and obtain distribution-expert models for each inherent distribution. In addition, VAEs can effectively model heterogeneous local data distributions and serve as density estimators. Such generative modeling of local distributions provides valuable distributional information, which can be leveraged to enhance the performance of downstream tasks in federated learning.

\subsection{Evaluation under Real-world Heterogeneity}

This subsection aims to examine whether the effectiveness of FedGMI extends from artificially constructed heterogeneous data to scenarios with naturally existing inherent distributions, and whether its behavior is consistent with that observed on MNIST and CIFAR-10.

As shown in Table.\ref{tab:emnist}, FedGMI consistently outperforms the baseline methods in terms of classification accuracy on both the uppercase and lowercase letter subsets. Moreover, a clear performance gap between the two inherent distributions can still be observed, indicating that FedGMI maintains differentiated modeling of the inherent components under natural heterogeneity. Furthermore, FedGMI achieves competitive client-associated accuracy, demonstrating that the learned distribution-expert classifiers effectively enhance personalized performance.

\begin{table}[ht]
    \centering
    \subfloat[Low Communication]{
    \resizebox{\columnwidth}{!}{%
    \begin{tabular}{@{}lccccccccc@{}}
    \toprule
     & \multicolumn{3}{c}{IFCA} & \multicolumn{3}{c}{FedSoft} & \multicolumn{3}{c}{FedGMI} \\ \midrule
     & \multicolumn{1}{l}{$c_0$} & \multicolumn{1}{l}{$c_1$} & \multicolumn{1}{l}{client} & \multicolumn{1}{l}{$c_0$} & \multicolumn{1}{l}{$c_1$} & \multicolumn{1}{l}{client} & \multicolumn{1}{l}{$c_0$} & \multicolumn{1}{l}{$c_1$} & \multicolumn{1}{l}{client} \\ \cmidrule(l){2-10} 
    \multicolumn{1}{l|}{upper} & {\ul 67.57} & 36.21 & 75.55 & {\ul 50.57} & 45.21 & 62.81 & {\ul\bf 86.83} & 76.09 & \textbf{93.71} \\
    \multicolumn{1}{l|}{lower} & {\ul 72.81} & 34.24 &  & 45.32 & {\ul 52.75} &  & 75.33 & {\ul \textbf{90.76}} &  \\ \bottomrule
    \end{tabular}%
    }
    }\\[4pt]

    \subfloat[High Communication]{
    \resizebox{\columnwidth}{!}{%
    \begin{tabular}{@{}llcclcclcc@{}}
    \toprule
        & \multicolumn{3}{c}{IFCA} & \multicolumn{3}{c}{FedSoft} & \multicolumn{3}{c}{FedGMI} \\ \midrule
        & $c_0$ & \multicolumn{1}{l}{$c_1$} & \multicolumn{1}{l}{client} & $c_0$ & \multicolumn{1}{l}{$c_1$} & \multicolumn{1}{l}{client} & $c_0$ & \multicolumn{1}{l}{$c_1$} & \multicolumn{1}{l}{client} \\ \cmidrule(l){2-10} 
    \multicolumn{1}{l|}{upper} & \multicolumn{1}{c}{62.67} & {\ul 85.51} & \textbf{95.88} & \multicolumn{1}{c}{{\ul 61.53}} & 56.53 & 73.82 & \multicolumn{1}{c}{{\ul \textbf{83.18}}} & 72.68 & 93.75 \\
    \multicolumn{1}{l|}{lower} & \multicolumn{1}{c}{74.86} & {\ul 90.38} &  & \multicolumn{1}{c}{61.13} & {\ul 65.36} &  & \multicolumn{1}{c}{75.59} & {\ul \textbf{92.86}} &  \\ \bottomrule
    \end{tabular}%
    }
    }
    \caption{\textbf{Accuracy of classification on EMNIST} : the notation for rows and columns is identical to that used in Table.\ref{tab:m_2}.}
    \label{tab:emnist}
\end{table}

Fig.\ref{fig:emnist_alpha} illustrates the estimated distribution proportions of uppercase letters. Similar to the CIFAR-10 case, FedGMI produces estimation that follows the trend of the ground-truth distribution variations but more conservative. In contrast, the outputs of the two baseline methods exhibit the same patterns as observed on the synthetically constructed datasets, showing less sensitivity to the fine-grained mixture structures present in the natural dataset.

\begin{figure}[h!]
    \centering
    \begin{tabular}{cc} 
        \subcaptionbox{MNIST, selection size = 5 \label{fig:emnist_alpha(a)}}{
            \includegraphics[width=0.42\linewidth]{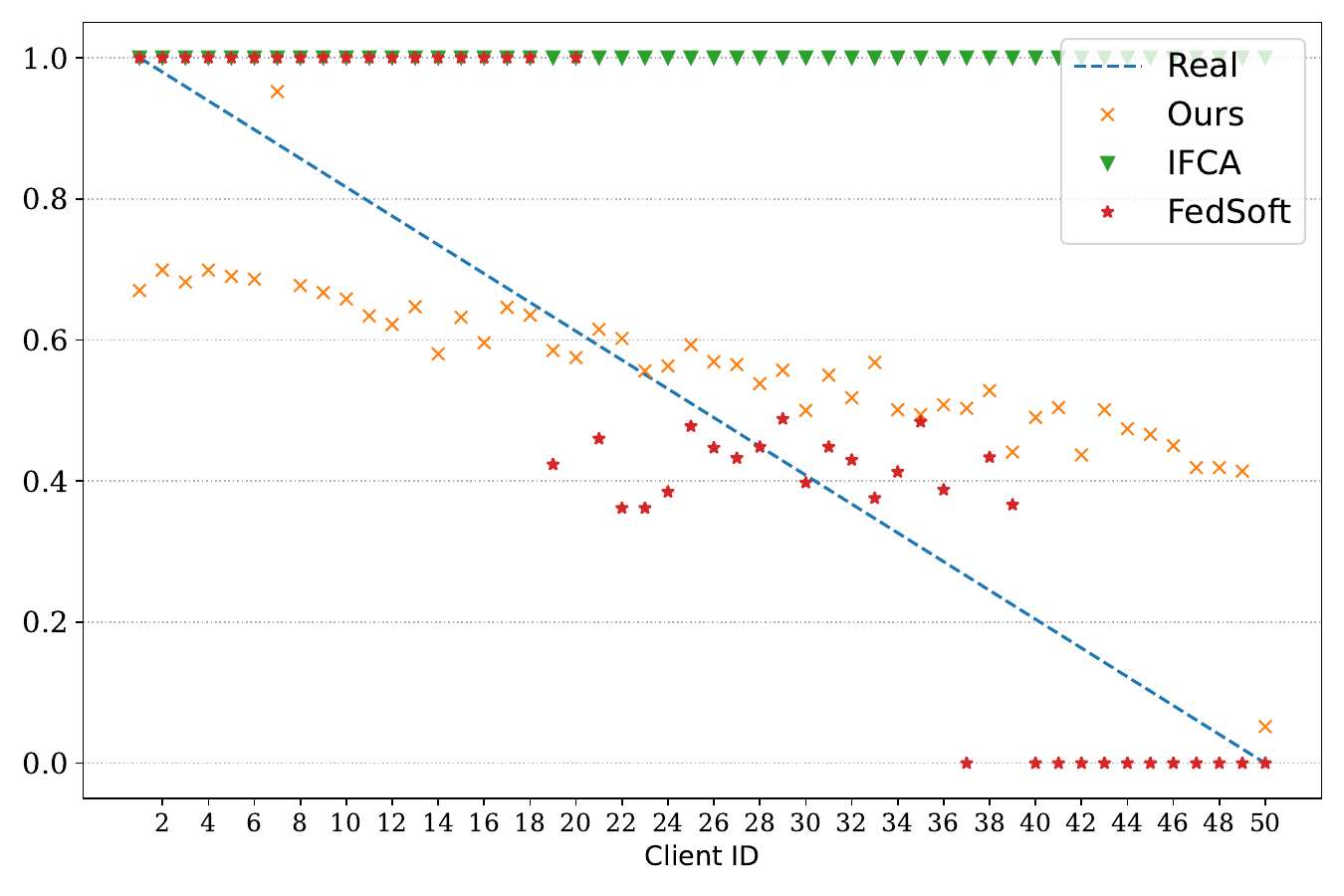}
        }
        \subcaptionbox{MNIST, selection size = 20 \label{fig:emnist_alpha(b)}}{
            \includegraphics[width=0.42\linewidth]{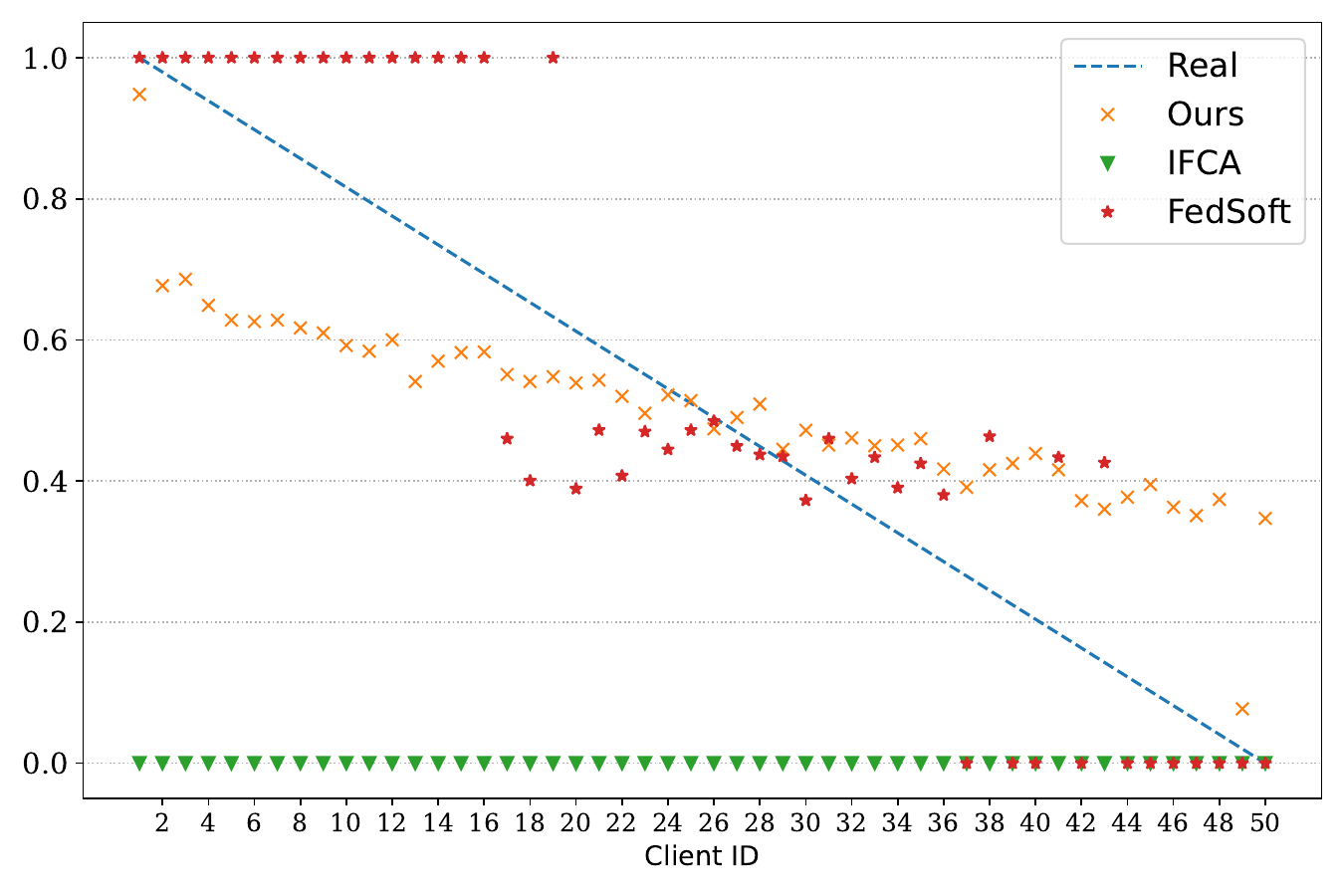}
        }
    \end{tabular}
    \caption{Estimation of the proportion of uppercase letters on EMNIST.}
    \label{fig:emnist_alpha}
\end{figure}

In summary, the results across all evaluations demonstrate that FedGMI effectively manages both synthetic and natural data heterogeneity. Most notably, the framework exhibits high consistency in its behavior across datasets, with its performance on real-world distributions mirroring the robust trends observed in artificially constructed scenarios. This stability across diverse mixture types and communication scales confirms FedGMI as a reliable approach for capturing latent probabilistic structures and enhancing classification performance in heterogeneous federated systems.

\section{Conclusion and Future Work}
In this paper, we studied Federated Learning in the context of the probabilistic mixture situation. We proposed FedGMI, a framework for learning the features of inherent distributions by leveraging generative models for sample-wise data partitioning. 
This framework provided comprehensive insight into the structure of local distributions, which explains the heterogeneity in probabilistic mixture situations. Theoretical analysis showed the convergence of the FedGMI algorithm and explained the necessity of stable initialization. Experimental results also demonstrated the ability of our framework to discriminate and represent the inherent distributions, as well as produce distribution-expert classifiers with higher accuracy. As a future direction, advanced generative models, such as diffusion models\cite{ref:diffusion}, could be integrated into our framework to enhance data partitioning and proportion estimation of data distribution. While such models hold promise due to their powerful generative capacity, developing efficient numerical techniques for this integration remains to be addressed. In addition, an important direction for future work is to incorporate adaptive clustering or dynamic model initialization strategies to automatically determine the number of inherent distributions \(M\) in real-world deployments.

\bibliographystyle{IEEEtran}
\bibliography{ref}


\end{document}